\documentclass[11pt]{article}

\usepackage[preprint]{acl}

\usepackage{times}
\usepackage{latexsym}

\usepackage[T1]{fontenc}

\usepackage[utf8]{inputenc}

\usepackage{microtype}

\usepackage{inconsolata}

\usepackage{graphicx}

\usepackage{url}
\usepackage[T1]{fontenc}
\usepackage{microtype}
\usepackage{hyperref}
\usepackage{url}
\usepackage{booktabs}

\usepackage{xcolor}
\usepackage{graphicx}
\usepackage[inline]{enumitem}
\usepackage{amsmath}
\usepackage{amsthm}
\usepackage{amssymb}
\usepackage{stmaryrd}
\usepackage{float}
\usepackage{caption}
\usepackage{subcaption}

\usepackage{lineno}
\usepackage[most]{tcolorbox}

\newtheorem{theorem}{Theorem}
\newtheorem{corollary}{Corollary}
\newtheorem{lemma}{Lemma}
\newtheorem*{remark}{Remark}

\newcommand{\nk}[1]{{\color{teal}{#1}}}

\newcommand{\todo}[1]{{\color{red}{TODO: #1}}}

\definecolor{darkblue}{rgb}{0, 0, 0.5}
\hypersetup{colorlinks=true, citecolor=darkblue, linkcolor=darkblue, urlcolor=darkblue}

\title{Flout at Your Own Risk: LLMs Struggle with Pragmatic Cooperativity Under Epistemic Asymmetry}

\author{
 \textbf{Hannah VanderHoeven},
 \textbf{Abhijnan Nath},
 \textbf{Nikhil Krishnaswamy}
\\
Situated Ground and Natural Language (SIGNAL) Lab, Department of Computer Science
\\
 Colorado State University / Fort Collins, CO 
\\
 \small{
   \texttt{\{Hannah.Vanderhoeven, Abhijnan.Nath, nkrishna\}@colostate.edu}
 }
}

\begin{document}
\maketitle
\begin{abstract}
Fruitful collaborations require effective communication of privately-held information.
The increasing use of LLMs as ``collaborators'' raises questions about the extent to which they exhibit these pragmatic capabilities.
In this paper, we perform a novel investigation into the pragmatic reasoning capabilities of LLMs in a multiparty collaborative task under partial information conditions.
We first formalize a notion of {\it collaborative epistemic asymmetry} and prove that fully actionable instructions in this context requires adherence to Grice's cooperative principle.
We then empirically assess various LLMs' abilities to act cooperatively as both speakers and listeners, including both prompting and post-training strategies.
Our results show that while LLMs exhibit certain pragmatic capabilities in collaborative settings, they still face challenges in pragmatic communication with incomplete information, and that Grice's maxims predict distinct failure modes.
\end{abstract}


\begin{figure*}
    \centering
    \includegraphics[width=.92\linewidth,clip,trim={20 0 20 0}]{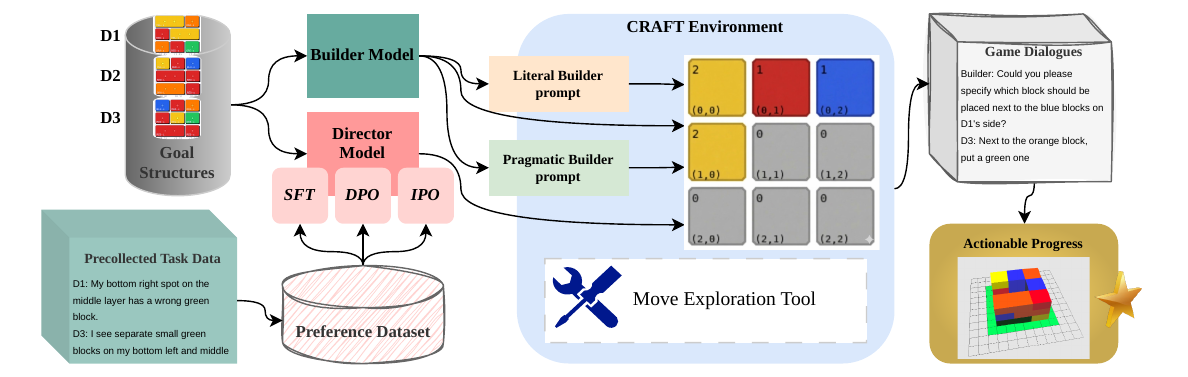}
\vspace*{-2mm}
    \caption{High level overview of our experimental pipeline.}
    \label{fig:overview}
\vspace*{-4mm}
\end{figure*}

\vspace*{-3mm}
\section{Introduction}
\label{sec:intro}
\vspace*{-2mm}

Successful collaboration rests on establishing {\it common ground} \citep{clark1991grounding,clark1996using}, by which collaborative groups can achieve outcomes greater than individual members acting alone \citep{boyd2021group}.
Inherent in this is the adjustment of communications to be {\it cooperative}---expressing the right amount of information, at the right time, in the right way \citep{grice1975logic}.

Recently, large language models (LLMs) have been incorporated into workflows across many domains, often as ``collaborators'' with humans, or with other AI systems in ``agentic'' pipelines \citep{butler2025microsoft,maslej2025artificialintelligenceindexreport}. 
Their fluent generation capabilities create an implicit assumption that LLMs and systems driven by them will exchange information and interpret collaborator needs and instructions in a way roughly equivalent to the cooperativity exhibited by humans \citep{zarriess2019know,guo2026embodied}. 
However, evidence regarding the pragmatic competence of LLMs is mixed at best \citep{nguyen2023language,jian2024llms,park2024multiprageval,sravanthi2024pub,ma2025pragmatics,ma2025vision,eisenstein2026mt, nath2026craft}. 
This question is particularly relevant in {\it multiparty} collaborations with more than 2 agents; agents may labor under "false assumptions" about others and how they communicate.

Through a novel theoretical formulation of {\it collaborative epistemic asymmetry} (CEA), this paper explicitly connects Gricean maxims to quantitative task outcomes in a challenging multi-agent, partially-observable collaborative reasoning task, which allows us to address the following research questions:
\begin{enumerate*}
    \item[\bf RQ1:] Under epistemic asymmetry, do common LLMs behave more like {\it literal} or {\it pragmatic} listeners?
    \item[\bf RQ2:] How much can greater pragmatic listening capabilities be elicited through prompting and encouraging exploration?
    \item[\bf RQ3:] How do Grice's maxims predict specific forms of pragmatic inference failure?
\end{enumerate*}
Fig.~\ref{fig:overview} shows an overview of our approach.

We show that the incomplete information setting induces substantial challenges for multi-agent coordination in collaborative tasks. While prompting, exploration, and stronger inference capacity can improve outcomes, we demonstrate where failure modes are predicted by agents' maxim flouting.

\vspace*{-2mm}
\section{Related Work}
\label{sec:related}
\vspace*{-2mm}


\citet{grice1975logic}'s \textit{maxims of conversation} defined rational assumptions of accurate, concise, relevant, and clear communication. 
While influential, they have proven challenging to operationalize computationally \citep{saygin2002pragmatics}, in part due to hard-to-quantify contextual nuances \citep{jurafsky2006pragmatics,floyd2025three}. 
The Rational Speech Act (RSA) framework \citep{goodman2013knowledge}---which assumes rational, informative speakers---has been used to model human recursive reasoning as pragmatic speakers and listeners \cite{monroe2015learning}.
RSA's \textit{pragmatic speaker} ($S_1$) maximizes the probability that a listener's {\it literal} interpretation ($L_0$) of their utterances will correctly infer the intended meaning.


For \textit{collaborators}, inferring referents and intended actions from each other's statements often necessitates {\it perspective-taking} to achieve common ground with interlocutors \citep{filipi2004perspective,cantiani2024perspective,davidson2025collaboration} while remaining anchored in one's own private knowledge \citep{katsos2023perspective}.
This rests on a {\it theory of mind} \citep{premack1978does}, which is challenging for LLMs \citep{ullman2023large,hu2025re}.
Multiple RSA- and MDP-based frameworks describe how agents maintain beliefs about the world and each other \citep{langlois2021rl,nath2025learning,maeda2026gesturing}, and retain robustness to non-cooperative or adversarial speakers \citep{gmytrasiewicz2005framework,gmytrasiewicz2020things}.
Dynamic epistemic logic (DEL) approaches explore necessary communicative strategies for collaborative inference \citep{van2011dynamic,van2014evidence,pacuit2017neighborhood,khebour2024common}. 

Modern LLMs are grammatically fluent language generators and linguistically competent ``listeners'' \citep{kibria2024functional,hardt2025sparks,ma2025pragmatics}, yet often miss meaning beyond words.
Prompting can help LLMs reason more pragmatically, but they remain limited in tasks with higher-order prerequisites like perspective-taking \citep{wilf2024think, just2025dipt, ma2025vision}.
\citet{estienne-etal-2025-collaborative} empirically improve performance in a partial-information reference game by optimizing cooperativity as a gain function.
In grounded multiagent LLM collaboration~\citep{andreas2016reasoning}, recent work  has addressed maze-solving~\citep{davidson2025collaboration} and 3D block building~\citep{wu2024coworkersmatterevaluatingcollaborative}, including partial-information settings such as CRAFT~\citep{nath2026craft}. We use CRAFT as our experimental testbed, but for a distinct set of research questions grounded in Grice's maxims using distinct experiments to analyze collaborative failures.



\vspace*{-3mm}
\section{Collaborative Epistemic Asymmetry: Actionability Requires Cooperativity}
\label{sec:theory}
\vspace*{-2mm}

In \textit{multi-turn, multi-agent collaborations} under \textit{partial information and observability}, parties with different background knowledge and capabilities must bring them together to solve a problem, {\it but no individual knows the solution at the outset} \cite{zhu2025multimodal,zhu2026distributed}. We refer to this as {\bf collaborative epistemic asymmetry} (CEA). This is commonplace in, e.g., ``jigsaw'' problems in classrooms \citep{perkins2011jigsaw} or challenging medical diagnoses \citep{poradzisz2019collaboration}. Here, if a collaborator communicates private task-relevant information inaccurately or unclearly, the others may be unable to act upon it, or the result may align with the speaker's intent. Thus an ``actionable'' utterance is not one that can be interpreted in just any way to permit {\it some action}; the action taken, if any, must be one consistent with the speaker's private information. This leads to a core insight of our work: {\it In collaborative task environments with epistemic asymmetry, cooperativity can be assumed only under a shared goal \cite{van2003quality}. Here, fully \textbf{actionable} utterances must be fully \textbf{cooperative}.}

\paragraph{Definitions} Given $N$ agents collaborating under epistemic asymmetry over $T$ timesteps, let $\mathcal{W}$ be the \textbf{world states}, of which $G$ is the \textbf{shared goal} state, in which the shared problem is solved.
No individual agent $i$ knows $G$ exactly, but each has private information $d_i$, drawn from a partition $\mathcal{I}_i$ of $\mathcal{W}$. $\mathcal{I}_i$ represents everything agent $i$ cannot rule out about $G$ given what they observe.
Let $\Phi$ be the space of atomic task-relevant propositions, $\varphi(w_t) \subseteq \Phi$ be the set of propositions true in world $w_t \in \mathcal{W}$, and $\varphi(d_i) \subseteq \Phi$ be the set entailed by private information $d_i$.
Let $\mathcal{U}$ be the space of \textbf{utterances} and $\varphi(u) \subseteq \Phi$ be the set of propositions entailed by $u \in \mathcal{U}$ (``what was said'').\footnote{\label{ftnt:union}Specifically, $\varphi(u) = \bigcup_{k=1}^K \varphi(u)_k$, for all $k$ interpretations of $u$, $\{\llbracket u\rrbracket_1\ldots\llbracket u\rrbracket_k\}$.}
Set-difference $\varphi(d_i) \setminus \varphi(w_t)$ denotes the gap between private knowledge and current public state.
\textbf{Collaborative epistemic asymmetry (CEA)} is parameterized by $\langle \mathcal{W}, G, \Phi, \mathcal{U}, \{\mathcal{I}_i\}_{i=1}^N\rangle$. Here, \citet{grice1975logic}'s 4 maxims can be considered constraints on utterances.
To an optimal listener---the Gricean idea pursued by \citet{goodman2016pragmatic}---a fully \textbf{cooperative} utterance satisfies all 4 constraints. 

\begin{theorem}[Actionability in CEA Settings Requires Gricean Cooperativity]
\label{thm:main}
An utterance $u$ is \textbf{fully actionable} under CEA iff (1) it asserts at least one proposition that is both true of the speaker's private information and not already true in the current world ($\llbracket u\rrbracket$ is {\bf admissible}, following \citet{van2011dynamic}) and (2) a listener can recover its intended content ($\llbracket u\rrbracket$ is fully ``{\bf recoverable}'' by a listener of capacity $\kappa$, adapting \citet{sperber2002pragmatics}'s notion of processing effort):

\vspace*{-6mm}
\begin{align}
\label{eq:admissability}
&\left(\forall p \in \varphi(u),\, p \in \varphi(d_i)\right) \wedge \left(\varphi(u) \not\subseteq \varphi(w_t)\right)\\
\label{eq:recoverability}
&~\text{and}~\frac{|\varphi(u) \cap \varphi(d_i)|}{|\varphi(u) \cup (\varphi(d_i) \setminus \varphi(w_t))|} = 1.
\end{align}

Under CEA, Gricean maxims are necessary for satisfying all above constraints, and thus an utterance $u$ is fully actionable iff it is fully cooperative. 
\end{theorem}

\begin{tcolorbox}[enhanced,breakable,title=Proof Sketch]
\hspace{4mm} \textit{\textbf{Quality} flouting:} If no proposition in $u$ reflects the
Speaker's $d_i$, then $p \in \varphi(d_i)$ fails for all $p$ and Eq.~\ref{eq:admissability} is false.

\hspace{4mm} \textit{\textbf{Relation} flouting:} If every proposition in $u$ is already settled at $w_t$ or off-task, $\varphi(u) \not\subseteq \varphi(w_t)$ fails for all $p$ and Eq.~\ref{eq:admissability} is false. 

\hspace{4mm} \textit{\textbf{Quantity} flouting} (padding, omission) and \textit{\textbf{Manner} flouting}
(ambiguity, disorder) expands $\varphi(u) \cup (\varphi(d_i) \setminus \varphi(w_t))$, the set containing both propositions interpretable from $u$ \textit{and} what must be said to close the gap between $d_i$ and $w_t$, beyond $\varphi(u) \cap \varphi(d_i)$, the set entailed by both $u$ and $d_i$, driving the ratio in Eq.~\ref{eq:recoverability} below 1.

\hspace{4mm} Thus under CEA, flouting any Gricean maxim violates one of conditions for full actionability. A full proof is in Appendix~\ref{app:proofs}.
\end{tcolorbox}
\vspace*{-2mm}

If Eq.~\ref{eq:admissability} is false \emph{no enactment of $u$ can move toward the goal}; the instruction has no true target at all. In contrast, Eq.~\ref{eq:recoverability} represents a graded outcome: recoverability is fully satisfied when Eq.~\ref{eq:recoverability} evaluates to 1, but may still be partially satisfied at values less than 1. This leads to a corollary of Theorem~\ref{thm:main}:

\begin{corollary}[Graded Recoverability]
\label{corr:recoverability}
Let $\delta_\kappa(u)$ denote the content a listener of capacity $\kappa$ recovers from $u$'s surface form, where $\kappa$ indexes the listener's inferential resources---analogous to recursive depth in RSA's listener hierarchy \cite{erk2022probabilistic,degen2023rational}.
In the limit $\kappa \to \infty$, $\delta_\kappa(u) \to \varphi(u)$: a sufficiently capable listener recovers the utterance's meaning, regardless of noise or competing readings admitted by its surface form.
\end{corollary}

CEA creates a noisy channel {\it a la} \citet{shannon1948mathematical}, in which cooperative utterances facilitate decoding of the original message. This yields a direct multi-agent reading of the Gricean maxims under CEA. Theorem~\ref{thm:main} makes a clear prediction: instructions that flout \textbf{Quality and Relation} should result in more \textit{failed outcomes}. Flouting \textbf{Quantity and Manner} should lead to outcomes that do not fail but are also not \textit{correct}.

Connecting cooperativity to actionability this way produces our second novel contribution: task metrics that reflect Gricean maxims. Collaborative agents are appropriately Gricean when their utterances contribute toward shared task progress.



\vspace*{-3mm}
\section{Methodology}
\label{sec:method}
\vspace*{-2mm}

\begin{figure}
  \centering
    \includegraphics[width=.85\linewidth,clip,trim={120 140 200 110}]{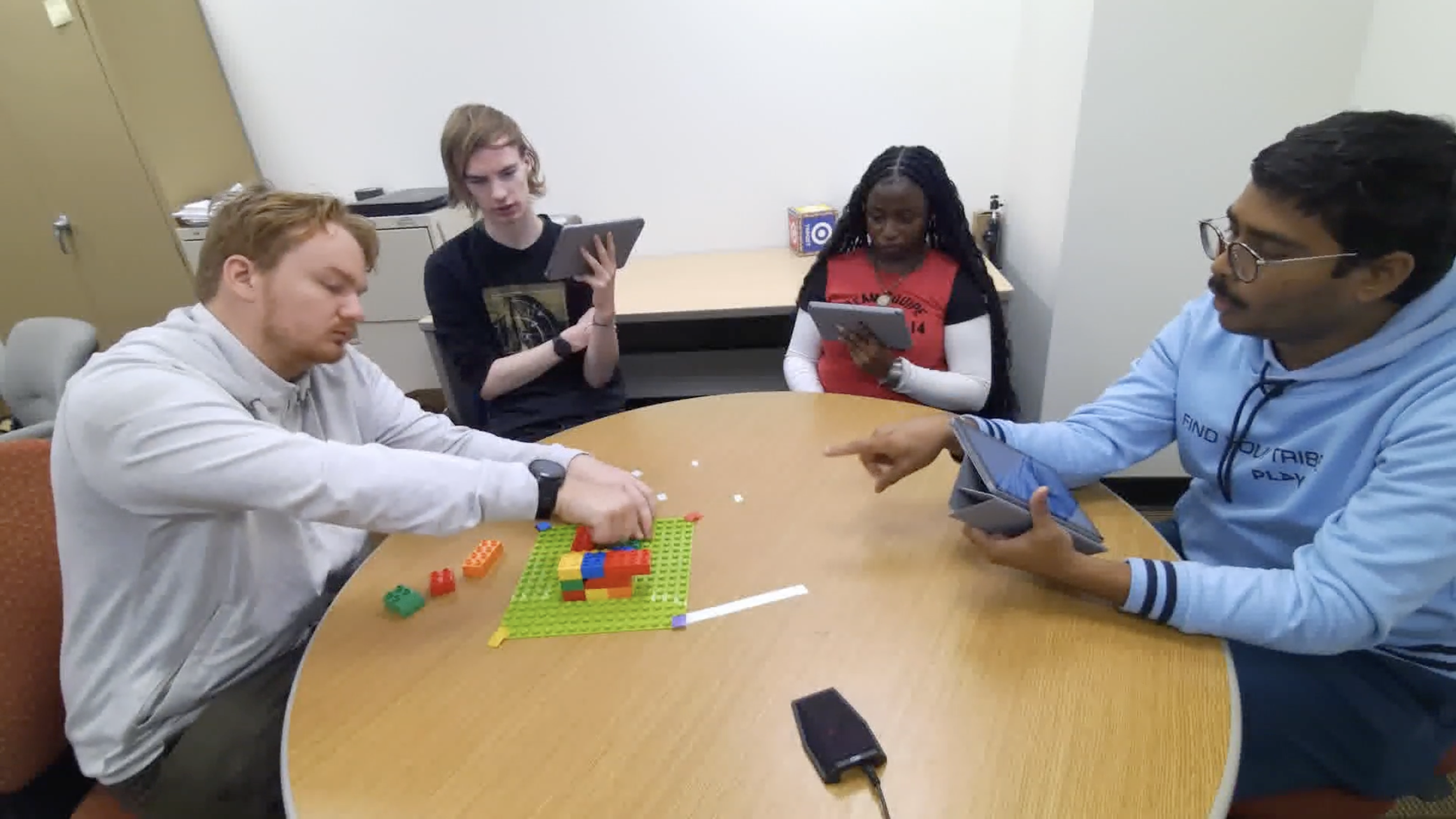}
    \centering
    \small
    \resizebox{\linewidth}{!}{
    \begin{tabular}{l l}
      \toprule
      \textbf{Agent} & \textbf{Utterance / Action} \\
      \midrule
      D1 & The long. \\
      B & This is yellow. \\
      D1 & Yeah, that one's yellow. \\
      B & So this is yellow. Which one's this color? Green? \\
      D3 & Yeah, the bottom one is green. \\
      B & \textbf{Action: REMOVE $\texttt{rs}$ from $layer~1$} \\
      \bottomrule
    \end{tabular}}
\vspace*{-2mm}
  \caption{Still from Distributed Partial Information Puzzle (DPIP) Lego task dataset \citep{zhu2026distributed} with accompanying dialogue.}
  \label{fig:dpip-ex3}
\vspace*{-4mm}
\end{figure}


We perform our examination in a setting derived from the Distributed Partial Information Puzzle (DPIP) Lego task \citep{zhu2025multimodal,zhu2026distributed} (Fig.~\ref{fig:dpip-ex3}). 
This task is a \textit{reference game} \citep{clark1986referring} where 3 {\it Directors} receive separate side views of a single structure made of large Lego blocks, and have to instruct a {\it Builder} to place the correct blocks in the correct places on a board to build it. All goal structures have a $3\times 3$ footprint (1 unit = 1 square Lego block) and are 3 layers high. The Directors cannot show their privately-held views to anyone else, and only the Builder can manipulate blocks and the board. All participants may see the board. This distinction in roles allows us to map the {\it Speakers} as defined in Sec.~\ref{sec:theory} to the Directors, and the {\it Listener} to the Builder. We can thus meaningfully speak of ``pragmatic'' or ``literal'' Builders who interpret Director utterances accordingly, or ``optimal'' Directors who attempt to maximize cooperativity in their utterances.

\citet{zhu2026distributed} present 10 annotated videos of humans performing this task. An example from the dataset (Fig.~\ref{fig:dpip-ex3}) demonstrates that the CEA setting produces the pragmatic behaviors predicted by Theorem~\ref{thm:main}. Here, Director 1 (D1) provides information that is both accurate and new to the Builder (B), satisfying Eq.~\ref{eq:admissability}, but underspecified with respect to the color of the bottom block. B's clarifying question prompts D3 to provide information that fully narrows the set of propositions entailed by D1's utterance (satisfying Eq.~\ref{eq:recoverability}), enabling B to \textit{act}: he removes a block from the board in anticipation of correcting the block beneath it. See Appendix~\ref{app:rsa} for an extended analysis based on the RSA framework.


\vspace*{-3mm}
\subsection{Task Environment}
\label{ssec:task-env}
\vspace*{-2mm}

The static nature of \citet{zhu2026distributed}'s original dataset makes it challenging to assess counterfactual outcomes, such as {\it what if the participants spoke more/less cooperatively?}\footnote{This is a known challenge of fixed datasets in multi-agent interaction, as noted by \citet{nath2025frictional}.} Thus, we adopt the associated CRAFT simulator and benchmark platform \citep{nath2026craft}.\footnote{\url{https://doi.org/10.5281/zenodo.18626417}}
\textsc{CRAFT} simulates the DPIP Lego task---including 3D “Lego” structure generation, conversion to Director-specific 2D partial views, and the Builder's move execution and validation logic---using a \textit{text-only} agentic framework that supports both open-weight and proprietary models in turn-based synchronous communication~\citep{li2023camel,nath2025collaborate,nath2025let}. 

LLMs can act as a Director (with privately held information) or the Builder.
Each turn, 3 Director instructions are generated in the context of the current board state, their private information, and the current dialogue history. 
Directors are randomly sampled with replacement. The Builder chooses one instruction to follow. CRAFT provides a {\bf move exploration tool} the Builder can use to assess move options before executing. This separately interprets each Director instruction and returns its estimated impact toward goal completion, providing a computational approximation of what the pragmatic listener $L_1$ should do. The Builder then executes the instruction of the most positive value.


The Builder can {\tt PLACE} or {\tt REMOVE} one block of specified color and size at a specified location, which updates the board state. It also generates a {\it confirmation} of its move in plain English. The move may \textbf{fail} if the Builder attempts something illegal (e.g., placing a block somewhere without a supporting block beneath it). The Builder may also request to {\tt CLARIFY} the instructions without executing. Moves (including fail and clarify outcomes) are appended to the dialogue history, which continues for a prespecified length.  Further technical details on CRAFT are in \citet{nath2026craft}. Further details on our specific usage are in Appendix~\ref{app:craft}.

Each ``game'' consists of a pre-generated goal structure $G$, which is partitioned to populate the Directors' private information $(d_1, d_2, d_3)$. One of 6 partial completion statuses is randomly pre-assigned: at the beginning of a game (start state $w \in \mathcal{W}$), the board can be {\it empty}, or the structure can have its {\it first 1 or 2 layers}, or a Director's wall (\textit{D1}, \textit{D2}, \textit{D3}) pre-completed consistent with the goal state. All agents have access to the current board state, $w_t$. Directors are assigned personality archetypes to increase lexical diversity in their utterances, using fixed roleplay prompts (Appendix~\ref{app:prompts}). These form a prior over $\mathcal{U}$, which is sampled from to create the dialogue history.  Thus a game maps neatly to the CEA specification (Sec.~\ref{sec:theory}). 

\vspace*{-3mm}
\subsection{Data Generation}
\label{ssec:generation}
\vspace*{-2mm}
Using \textsc{CRAFT}, we generated 2 games each ($T=20$ turns) over 100 pre-generated goal structures for sampling preference data. Separate GPT-4.1-mini instances layed the Builder and the Directors.\footnote{Initial experimentation showed GPT-4.1-mini to best balance lexical diversity, overall task progress in 20 turns, and cost of running 100 games (Appendix~\ref{app:model-selection}).} The Builder agent used the move exploration tool to maximize progress.  From this data, we constructed a preference dataset $\mathcal{D} = \{(x,y_w,y_\ell)\}_{i=1}^N$ where each sample consisted of a dialogue history (context) $x$ terminating before a given turn $t$, the most actionable Director utterance $y_w$ (defined as the utterance the Builder acted upon at $t$) and a suboptimal utterance $y_\ell$ from the same turn. This data (6,201 samples) was used to align open-weight models to act as Directors. Twenty more goal structures were held out for evaluation (Sec.~\ref{sec:exp}).

\vspace*{-3mm}
\subsection{Metrics}
\label{ssec:metrics}
\vspace*{-2mm}

One of our core insights in this work is that in a CEA setting, even a legal move made in response to a Director utterance may not be {\it correct} w.r.t. the final goal structure if the instruction was not fully cooperative in communicating its intended referent  (Sec.~\ref{sec:theory}). Therefore we calculated a ``correctness score'' (0--6) for each turn based on how the utterances and actions therein furthered task completion and group common ground. Starting from a floor of 0 for a {\bf failed} move (non-actionable instruction), 1 point was assigned if the Builder requested clarification (signaling un-enactable but partially informative instructions). At least 2 points was assigned for a successful move, with 1 more point each for
\begin{enumerate*}[label=\arabic*)]
    \item placing/removing the right block at/from the right position in the goal structure,
    \item placing/removing the right block at/from the right position in the Director's private view, 
    \item if all blocks currently in the goal structure were correctly placed\footnote{To avoid unduly penalizing for errors made in previous turns, a discount factor of $\gamma = 0.5$ was applied to this metric. The correctness value was incremented by 1 if this metric was {\tt True}, and if {\tt False}, by $(1 - (\gamma^{p-1}))$ where $p$ is the number of consecutive previous turns where the overall structure remained incorrect due to an existing uncorrected error.}, and
    \item if utterances in the turn agree on the move, according to an LLM-Judge (prompt in Appendix~\ref{app:prompts}).
\end{enumerate*}

We also calculate the objective {\bf Completion \%} of the game (percentage of target blocks that are correctly placed). As some games started with a non-empty board, we normalized metrics relative to their values at $t=0$ and report deltas.


\vspace*{-3mm}
\section{Experiments}
\label{sec:exp}
\vspace*{-2mm}

We tested \textit{closed models} (\textbf{GPT-5.4-mini}, \textbf{GPT-4.1-mini}) as the Builder, while closed or \textit{open-weight models} (\textbf{Qwen3.5-9B-Instruct}) could be Directors.\footnote{Full prompts in Appendix~\ref{app:prompts}, hyperparameters in Appendix~\ref{app:hyperparams}, additional results with Builders such as \textbf{GPT-5.6-Luna}, \textbf{GPT-4o-mini} and Directors such as \textbf{Qwen2.5-7B-Instruct} and \textbf{Llama 3.1-8B-Instruct} in Appendix~\ref{app:model-selection}.} Goal structures were randomly pre-split among the 6 start states except where noted. Each game used pre-assigned Director archetypes and 20 turns to keep conditions consistent.

\paragraph{Empirical Performance without CEA}

To establish what the ceiling on LLM performance should be \textit{without} CEA, we ran games with the 20 test structures under two conditions that allowed agents access to private views normally be hidden to them. The first of these uses a single Builder agent only (no Directors) who has access to all 3 normally private views of the goal state, as well as the CRAFT ``Oracle'' \cite{nath2026craft} which provides 3 legal moves at each step, grounded to the real goal structure. This fully collapses CEA into a single-agent condition with full information about the goal. The second condition introduces the Directors in addition to the Builder and the requirement for them to provide instructions, making it genuinely multi-agent, but the agents all have access to the 3 private views. This breaks information separation but introduces the communication requirement between the participants: utterances may be noisy and inaccurate as discussed in Sec.~\ref{sec:theory}, and forces the Builder to choose utterances to act upon while performing its own inference using CRAFT's move exploration tool. Table~\ref{tab:full-info} shows the ceiling results.

\begin{table}[h!]
\vspace*{-2mm}
  \centering
  \resizebox{\columnwidth}{!}{\begin{tabular}{lcccc}
\toprule
 & \multicolumn{2}{c}{\textbf{GPT-4.1-mini}} & \multicolumn{2}{c}{\textbf{GPT-5.4-mini}} \\
\cmidrule(lr){2-3} \cmidrule(lr){4-5}
\textbf{Condition} & Compl. \% $\Delta$ & Compl. \% & Compl. \% $\Delta$ & Compl. \% \\
\midrule
Single Builder & $\mathbf{62.4_{\pm 6.0}}$ & $\mathbf{92.2_{\pm 1.9}}$ & $\mathbf{63.5_{\pm 5.8}}$ & $\mathbf{93.4_{\pm 2.0}}$ \\
Multi-Agent & $10.1_{\pm 4.6}$ & $39.9_{\pm 2.6}$ & $16.1_{\pm 3.4}$ & $45.9_{\pm 4.1}$ \\
\bottomrule
\end{tabular}}
\vspace*{-2mm}
  \caption{Full-information performance with GPT-4.1-mini and GPT-5.4-mini agents.}
  \label{tab:full-info}
\vspace*{-4mm}
\end{table}

The single Builder with full information almost fully completed all structures, regardless of start condition. When the Directors were introduced and required to share information, performance in 20 turns dropped substantially with a mean of 10.1--16.1\% progress beyond the game start. These empirical ceilings show that
\begin{enumerate*}[label=(\arabic*)]
    \item the DPIP task is not natively too hard for agents with access to full information, but
    \item the noisy inter-agent communication channel introduces a pronounced challenge, which full CEA with information separation sharpens, as the remaining experiments show.
\end{enumerate*}

\vspace*{-2mm}
\paragraph{Prompting Experiments}
This set of experiments addressed {\bf RQ1} and {\bf RQ2} by {\bf prompting} the Builder LLM (a closed model) to act as a {\it literal} ($L_0$) or {\it pragmatic} listener ($L_1$) as in the RSA framework. 
These states were invoked through simple prompt ``infixes'' following e.g., \citet{ward2023honesty,nath2025learning}. 
The ``Literal Builder'' stated ``{\it Read each director's message at face value}'', while the ``Pragmatic Builder'' was prompted to think through a set of questions adapted from \cite{chenail2014}, relating to how well the information presented in an utterance aligned with Gricean maxims. These conditions were compared to a baseline condition with no explicit instructions on how to interpret Director utterances, aside from the default CRAFT system prompt. The 20 test structures were each run twice.

\vspace*{-2mm}
\paragraph{Tool-Calling Experiments}
We further compared the Builder acting with the \textbf{move exploration tool} vs. without. This tested how Builder exploration improved selection of actionable utterances and therefore outcomes, addressing {\bf RQ2}.

\vspace*{-2mm}
\paragraph{Post-Training Experiments}
Here, open-weight models acting as Directors were \textbf{post-trained} using SFT, DPO \citep{rafailov2024direct} and IPO \citep{azar2024general} against the preference dataset $\mathcal{D}$ (Sec.~\ref{ssec:generation}). Intuitively, Directors could be made more robust to a  Builder behaving more like a literal listener ($L_0$) by aligning their utterances toward more actionable examples as captured in the preference dataset. Each test structure was run \textbf{once} in these experiments.


\begin{figure*}[h!]
\vspace*{-2mm}
    \centering
    \includegraphics[width=.91\linewidth]{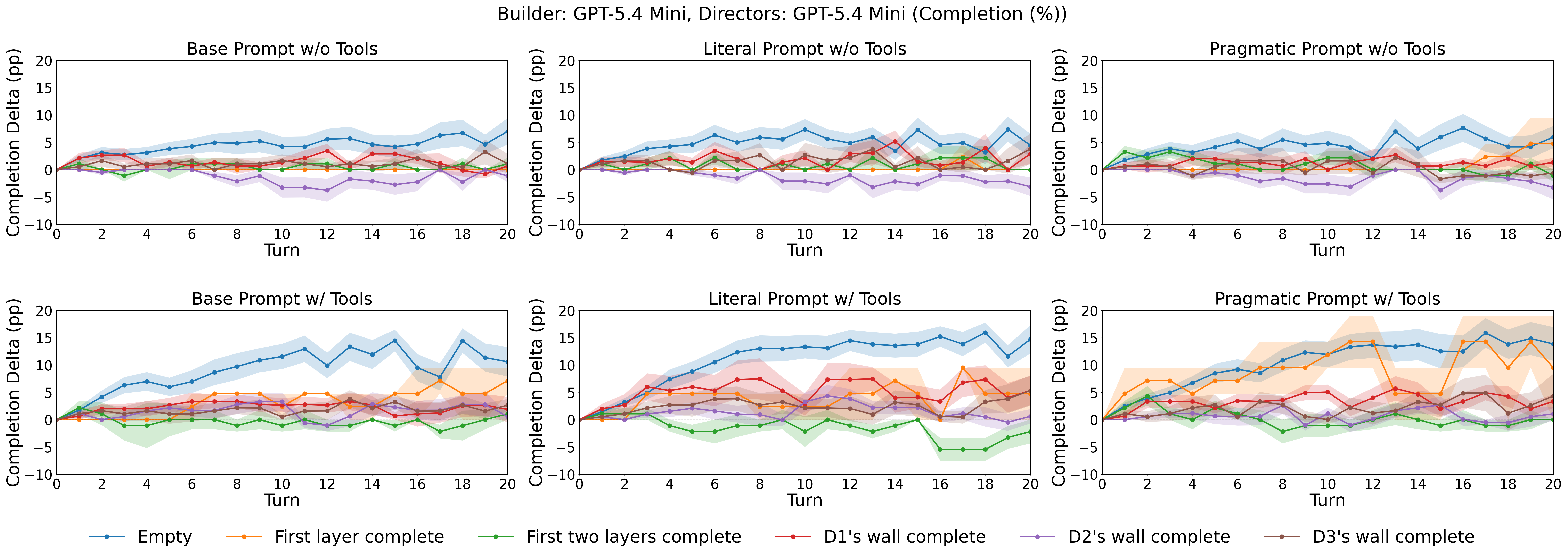}
\vspace*{-2mm}
    \caption{Prompting and tool-calling performance with GPT-5.4-mini as Directors and Builder. Shading reflects std. error of the mean.}
    \label{fig:main_prompting_results}
\vspace*{-2mm}
\end{figure*}

\vspace*{-3mm}
\section{Main Results}
\label{sec:results}
\vspace*{-2mm}

\begin{table}[h!]
\vspace*{-2mm}
  \centering
  \resizebox{.98\linewidth}{!}{%
\begin{tabular}{lcccc}
\toprule
 & \multicolumn{2}{c}{\textbf{Director: GPT-4.1-mini}} & \multicolumn{2}{c}{\textbf{Director: GPT-5.4-mini}} \\
\cmidrule(lr){2-3} \cmidrule(lr){4-5}
\textbf{Prompt} & Compl. \% $\Delta$ & $p$ & Compl. \% $\Delta$ & $p$ \\
\midrule
Base w/ tool & $5.8_{\pm 1.6}$ & --- & $7.9_{\pm 1.6}$ & --- \\
Base w/o tool & $1.4_{\pm 1.3}$ & --- & $5.4_{\pm 1.3}$ & --- \\
Literal w/ tool & $6.4_{\pm 1.4}$ & 0.46 & $5.9_{\pm 1.8}$ & 0.30 \\
Literal w/o tool & $2.3_{\pm 1.5}$ & 0.34 & $4.5_{\pm 1.3}$ & 0.39 \\
Pragmatic w/ tool & $\mathbf{8.7_{\pm 1.6}}$ & \underline{0.05} & $\mathbf{8.4_{\pm 1.7}}$ & 0.74 \\
Pragmatic w/o tool & $-0.1_{\pm 1.5}$ & 0.13 & $4.7_{\pm 1.2}$ & 0.51 \\
\bottomrule
\end{tabular}}
\vspace*{-2mm}
    \caption{
    Comparison (Mean$_{\pm \text{SEM}}$) of Completion \% $\Delta$ across Directors (Builder: GPT-4.1-mini, $N=40$). $p$ come from a paired $t$-test with analogous base prompt. Significant values underlined at threshold of $p=0.05$.
    }
  \label{tab:builder_comparison_combined}
\vspace*{-4mm}
\end{table}

\noindent
Table~\ref{tab:builder_comparison_combined} shows completion since game start under various prompting/tool-calling conditions using GPT-4.1-mini and GPT-5.4-mini as Directors and GPT-4.1-mini as the Builder. Full CEA with strict information separation poses the hardest challenge yet for agentic coordination. Group performance is cut in half from the collaborative full information condition in Table~\ref{tab:full-info}, and by a factor of 8 or more when compared with the single-builder condition.

With a fixed GPT-4.1-mini builder, the GPT-4.1-mini director with pragmatic prompting and tool calling shows a statistically significant improvement over the baseline prompting condition ($p\approx0.05$). However, with GPT-5.4-mini as the director, pragmatic and base conditions are statistically indistinguishable ($p\approx0.74$). The \textit{literal} prompt underperforms baseline by 1--2\% points, but not to a significant level. OTS GPT-4.1-mini acts more like a {\it literal listener} to another instance of itself, but with the GPT-5.4-mini director, the pragmatic prompting is unnecessary for the builder model ({\bf RQ1}). This in turn suggests that GPT-5.4-mini has more pragmatic \textit{speaking} abilities, such that a GPT-5.4-mini utterance received by GPT-4.1-mini is already optimized to be correctly interpreted by a literal listener. This would align with predictions of the RSA framework.

Using the move exploration tool helps performance under all prompting conditions (e.g, when GPT-5.4-mini plays both Directors and Builder, $p\approx 0.03$ in the base prompting condition, $p\approx 0.02$ in the literal condition, and $p<0.001$ in the pragmatic condition; with GPT-4.1-mini vs. GPT-4.1-mini, $p<0.001$ across all prompting conditions; see Tables~\ref{tab:toolsGpt41} and~\ref{tab:toolsGpt54} in Appendix~\ref{app:additional} for more,
suggesting that allowing the Builder to explore more interpretations of Director utterances helps it more optimally reason about Directors' own interpretations of $u$ and underlying goals ({\bf RQ2}). 
The tool gives the model specific machinery for pragmatic reasoning.

Fig.~\ref{fig:main_prompting_results} breaks down Completion \% $\Delta$ over turns and by start state. We see the same improvement with move exploration and the Pragmatic Builder. The strongest effect is in games starting with an empty board, which show steady progress toward the goal helped further by pragmatic prompting and exploration. 
The combination also helps other starting conditions, such as starting from a complete first later. In other prompting, tooling, and starting conditions the group makes little progress or even regresses by undoing some of the correctly pre-completed structure (see Sec.~\ref{sec:disc} for more).



\vspace*{-3mm}
\section{Discussion: All Flouts Are Not Created Equal}
\label{sec:disc}
\vspace*{-2mm}

The evidence toward RQ1 and RQ2 indicates that ``pragmatic'' LLM behavior under epistemic asymmetry emerges from an interplay between speakers and listeners rather than in isolation. Therefore, to address RQ3, we introduced post-trained open-weight models as Directors to examine where pragmatic inference failures occur.


Fig.~\ref{fig:violations} shows analysis of failure modes according to each Gricean maxim.
Theorem~\ref{thm:main} (Sec.~\ref{sec:theory}) predicts that flouting Quality and Relation (Eq.~\ref{eq:admissability}) admits no valid target action and results in higher move failure rates\footnote{For all failure rate analyses, Builder {\tt CLARIFY} requests are considered failed moves.}, and that flouting Quantity and Manner (Eq.~\ref{eq:recoverability}) may admit a valid target that is difficult to recover, resulting in lower correctness.

\begin{figure*}[h!]
\vspace*{-2mm}
    \centering
     \begin{subfigure}[b]{\textwidth}
        \centering
        \includegraphics[width=0.91\textwidth]{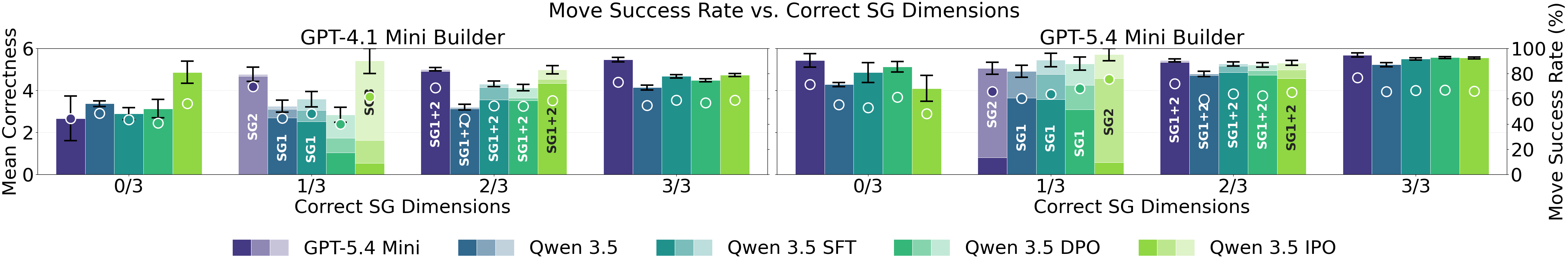}
        \label{fig:violations-quality}
    \end{subfigure}
    \begin{subfigure}[b]{\textwidth}
        \centering
        \includegraphics[width=0.91\textwidth]{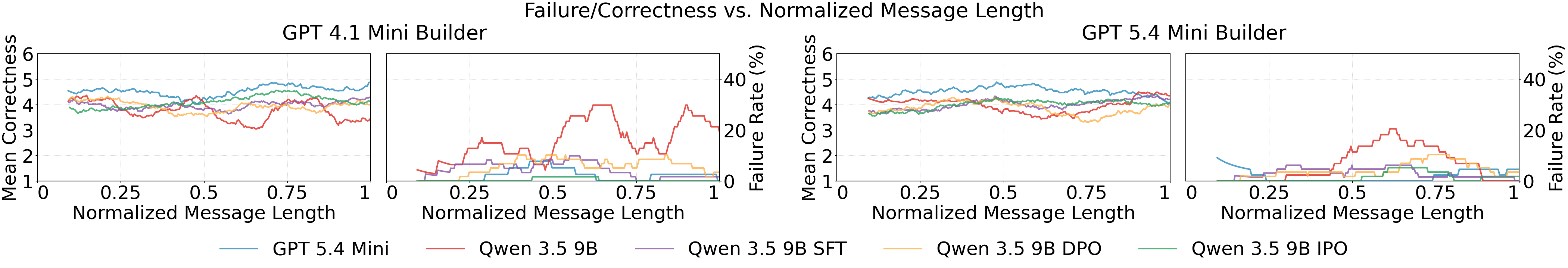}
        \label{fig:violations-quantity}
    \end{subfigure}
    \begin{subfigure}[b]{\textwidth}
        \centering
        \includegraphics[width=0.91\textwidth]{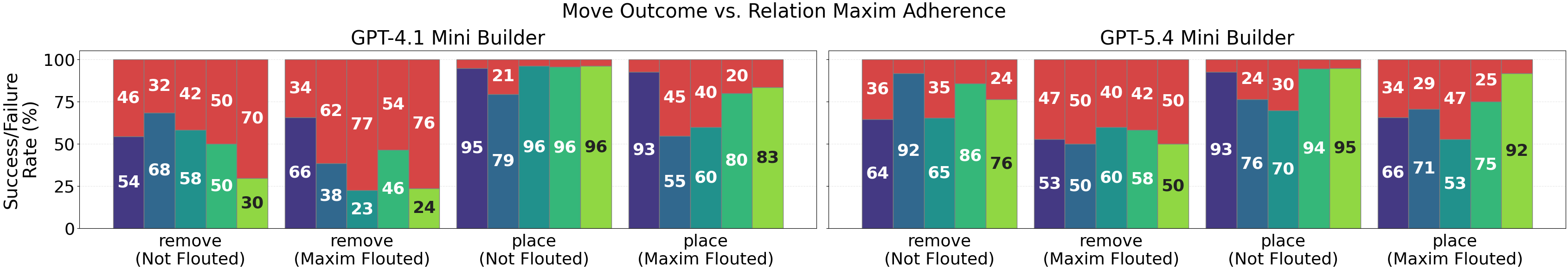}
        \label{fig:violations-relation}
    \end{subfigure}
    \begin{subfigure}[b]{\textwidth}
        \centering
        \includegraphics[width=0.91\textwidth]{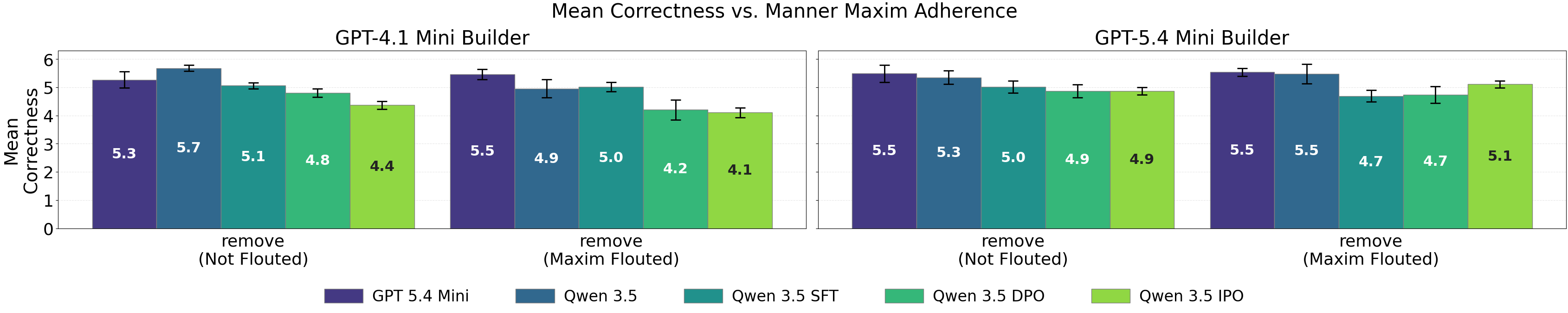}
        \label{fig:violations-manner}
    \end{subfigure}
\vspace*{-6mm}
    \caption{\label{fig:violations}Top to bottom: 1) Move success rate vs. correct grounding dimensions in Director utterances. 
    \textbf{Shades} in columns 1/3 and 2/3 indicate the proportion of total correct SG elements belonging to each element or subset. \textbf{Circles} indicate mean move correctness.
    2) Move failure rate/mean correctness vs. normalized message length by Director model. 3) {\tt PLACE} and {\tt REMOVE} outcomes vs. maxim of relation adherence. 4) {\tt REMOVE} outcomes vs. maxim of manner adherence. Builder uses Pragmatic prompting and move exploration.}
\vspace*{-4mm}
\end{figure*}

\vspace*{-2mm}
\paragraph{Quality Flouting} A Director describing their wall view inaccurately flouts the maxim of quality. 
The utterance is {\it inadmissable} (Theorem~\ref{thm:main}) and should lead to more likely move failure. 
A GPT-4.1-mini LLM-Judge examined Director utterances for accurate \textit{color}, \textit{size}, and \textit{layer position} representation. 
Fig.~\ref{fig:violations}, Row 1 shows move success rate vs. accurate spatial grounding (SG) dimensions in Director utterances. 
Within-Director success increases with correct grounding ($p<0.001$).
SEM decreases, indicating more stability with better utterance quality. 
Director alignment correlates with improved utterance quality over base models.
SG element 2 (size) appears the most important, sufficient alone or with SG1 (color) to facilitate 60--70\% move success. 
This may be because correctly placing a large block fills \textit{\textbf{2}} cells and substantially advances completion. 
Circles in the bars show mean correctness of executed moves; although GPT-5.4-mini has a higher success rate, executed move correctness is flat, a consequence of Corollary~\ref{corr:recoverability}: a listener of greater capacity can recover {\it a} move from a quality-flouting utterance, but one inconsistent with the speaker's true private information.




\vspace*{-2mm}
\paragraph{Quantity Flouting} Row 2 shows move failure rate and correctness vs. normalized Director utterance length (in words) using a rolling average with a window of 20\% of total samples. 
Longer Qwen base model utterances tend to increase GPT-4.1-mini failed moves and decrease correctness. 
This shows sensitivity to flouting the maxim of quantity, and that overinformativeness (proxied by length) adversely degrades the recoverability of the message ($p\approx0.007$ for GPT-4.1-mini Director).
Director alignment increases correctness and decreases move failure over the base model, achieving numbers comparable to a GPT-5.4-mini Director.
The GPT-5.4-mini Builder is more stable across Directors, keeping correctness higher and failure rate lower ($p\approx0.428$ correlation with normalized length); stronger Builder inference capacity $\kappa$ entails greater recoverability of $\varphi(u)$, as predicted by Theorem~\ref{thm:main}.
See Appendix~\ref{app:quantity} for more.

\vspace*{-2mm}
\paragraph{Relation Flouting} 
Row 3 shows {\tt PLACE} and {\tt REMOVE} outcomes in cases where a Director's wall is already complete, rendering any instruction pertaining to that wall irrelevant, but its ``owner'' gives one anyway. 
If the Builder acts upon the utterance from a different Director, relation adherence is preserved, otherwise it is flouted. 
Flouting the relation maxim strongly correlates with move failure ($p\approx 0.007$ for the GPT-5.4-mini Builder, GPT-4.1-mini $p\approx 0.031$): the ``wall owner'' pushes for changes to a complete wall, which are not admissible and fail as Theorem~\ref{thm:main} predicts.
This explains in part why games that started with non-empty boards showed inferior progress (Sec.~\ref{sec:results}, Fig.~\ref{fig:main_prompting_results}).
{\tt REMOVE}s are more challenging moves to explain and execute; GPT-5.4-mini with its stronger ``pragmatic listening'' handles {\tt REMOVE}s better under cooperative conditions, but cannot overcome irrelevant information which admits no valid move target. 

\vspace*{-2mm}
\paragraph{Manner Flouting} For an explicit manner flouting condition, we consider cases where a Director provides two instructions in one turn. This creates multiple readings and ambiguity, flouting the maxim of manner. As in relation, the Builder may choose a different Director's utterance entirely, preserving manner adherence, or one of the multiple interpretations of the flouting Director's utterances. Row 4 shows the effect of such manner flouting on move outcomes. Within-Builder effect is virtually non-existent, indicating already strong disambiguation capacity even in older LLMs. The GPT-5.4-mini Builder slightly increases correctness when manner is flouted over the GPT-4.1-mini Builder, indicating its greater capacity for recovery as predicted by Theorem~\ref{thm:main}. This difference is evident (3/5 Directors, 1 tie) but because GPT-4.1-mini is already strong enough at this, it is not statistically significant ($p\approx0.192$), indicating an area in which, unlike relation flouting, LLMs already have pragmatic listening capability under CEA.



\vspace*{-3mm}
\section{Conclusion}
\label{sec:conc}
\vspace*{-2mm}

In this paper, we showed that epistemic asymmetry introduces a strong coordination challenge for multi-agent LLM teams. 
We formalized a notion of collaborative epistemic asymmetry (CEA) that explicitly connects the actionability of task-centered utterances to satisfaction with Grice's maxims of conversation, showing that under certain assumptions, flouting the maxims negatively impacts actionability and thus collaborative task progress.
Empirical evidence bolsters our theoretical formulation: exploration consistently helps, closing the gap between strict information separation and multi-agent conditions without information separation; pragmatic speaking and listening behaviors emerge from the interaction between Directors and Builder, noise is more recoverable than low quality or irrelevant information (Sec.~\ref{sec:disc}).

Similar questions have been examined in human collaborations and single AI-human collaborations, but never to our knowledge in multiagent settings. Other recent work (e.g., \citet{eisenstein2026mt}) identified shortcomings in LLM reasoning in private information contexts; we identify unique challenges in multiagent settings where the asymmetric nature of available information creates unique conditions that challenge LLM reasoning.  

Our work also indicates the order of severity of each challenge: according to significance of flouting effect (Sec.~\ref{sec:disc}) multiagent teams under CEA degrade most with flouting of \textit{quality}, then \textit{relation}, \textit{quantity}, and finally \textit{manner}; this ordering recapitulates the admissability/recoverability split predicted by Theorem~\ref{thm:main}. As LLMs improve as collaborators, we may expect them to become better still at recoverability but remain challenged by inaccurate or irrelevant information.

Future work may explore novel optimization to enable Builder/Listener models to adaptively seek information from specific speakers, online Director optimization using methods like PPO or GRPO, or intentional Gricean \textit{violations} by introducing perturbed private information or adversarial agents.

\section*{Limitations}
While we make no specific claims about cooperativity in human-LLM interactions, our results suggest that even in LLM-LLM interactions rendered in human-readable language, LLMs suffer from pragmatic failures in both speaking and listening, but also display some pragmatic listening ability to recover from speakers' flouting of certain maxims. 
This work applies Grice's maxims and a novel interpretation of them as a framework to analyze the dynamics of natural language based multi-agent LLM collaboration, 
but also provides a framing through which NLP techniques may be used to analyze human communication.

Gricean maxims are commonly used to describe communicative behavior, and flouting in particular helps reveal the intent behind an utterance. 
While \citet{grice1975logic}'s original intent was to describe principles followed by people in conversation, with the underlying intent that implies, it is common in NLP literature to use ``flouting'' to describe behavior and motivate experiments involving LLMs \cite{sravanthi-etal-2025-understand, sato-etal-2025-pragmatic}, even as it is widely acknowledged that these models lack any genuine intent behind these ``simulated flouting'' behaviors \cite{hu-etal-2023-fine, bojic2025gpt4, RAPPA2026101489}. 
This work follows this paradigm, in which ``flouting'' is used as the framing for communicative trends seen in LLMs but does not attribute it to any intent on the part of the agent. 
In our work, all agents seek collaborative success rather than adversarial outcomes, so we focus on floutings of Gricean maxims such as accidental rule-breaking or omission of important information, rather than intentional {\it violations} to hide meaning.

\section*{Ethics Statement}
Due to human diversity of expression and the human tendency to behave in ways that challenge even theoretically rigorous and empirically validated AI systems, our results have bearing on human-LLM interactions and suggest that assumptions of current LLMs' pragmatic competence in collaborative settings should be treated cautiously or even skeptically. These conclusions would need to be tested in real multiparty human-AI collaborations in order to assess the precise effects of human diversity on LLM performance, as well as susceptibility to factors such as bias against non-normative modes of communication. Additionally, because of the focus of our work on agent-agent interactions in collaborative tasks with measurable outcomes, ``flouting'' as used in the context of our experiments focuses less on social nuance in human communication, such as irony, hyperbole, or the use metaphors to imply meaning, in favor of measuring instances where LLMs break the rules outlined by the maxims of conversation without direct intent to imply additional meaning. These social nuances are often situation- and culture-dependent and so agentic AI settings remain inappropriate settings to study and make claims about them.

\section{Acknowledgments}
This material is based in part upon work supported by Other Transaction award 1AY2AX000062 from the U.S. Advanced Research Projects Agency for Health (ARPA-H) Platform Accelerating Rural Access to Distributed Integrated Medical Care (PARADIGM) program. Approved for public release, distribution unlimited. Views expressed herein do not reflect the policy or position of the Department of Health and Human Services, or the U.S. Government.

\bibliography{colm2026_conference}

\appendix
\section{Proof of Theorem~\ref{thm:main}}
\label{app:proofs}
Recall that $\varphi(u) = \bigcup_{k=1}^K \varphi(u)_k$ (Sec.~\ref{sec:theory}, footnote~\ref{ftnt:union}). For simplicity, let us refer to the set-differential gap between director private knowledge and the current world state $\varphi(d_i) \setminus \varphi(w_t)$ as $R(w_t)$ (``what remains to be communicated at $w_t$ to fully satisfy director private knowledge''). Since a director usually does not articulate {\it all} of their remaining private information at every timestep, in practice a filtration $\mathbf\Gamma_t$ denotes the unarticulated remainder of said gap after turn $t$'s exchanges: $\mathbf\Gamma_t := \varphi(d_i) \setminus \varphi(w_t) \setminus \bigcup_{t'\leq t}\varphi(u_{t'})$. Therefore $R(w_t) := \mathbf\Gamma_{t-1}$. Eq.~\ref{eq:recoverability}'s denominator uses the gap $\varphi(d_i)\setminus\varphi(w_t)$ for legibility. The refined definition of filtered $R(w_t):=\Gamma_{t-1}$ additionally accounts for information already articulated in prior turns, and is taken as the operative definition for the remainder of the proof.
The specific decoding or parsing function of a listener of capacity $\kappa$ is denoted $\delta_\kappa(\cdot)$, as in Corollary~\ref{corr:recoverability}.

\begin{proof}
Let us first explicitly define the constraints on an utterance $u$ entailed by \citet{grice1975logic}'s cooperative maxims in the collaborative epistemic asymmetry (CEA) setting (see Sec.~\ref{sec:theory}).
Given an utterance $u$, speaker private information $d_i$ and current world $w_t$:

\begin{itemize}
    \item {\bf Quality}: speaker $S_i$ should only assert what they believe to true given $d_i$, their own information state. Quality is {\bf fully satisfied} iff $\varphi(u)$, the proposition asserted by $u$, is consistent with $d_i$, s.t. $\forall p \in \varphi(u),\, p \in \varphi(d_i)$.
    \item {\bf Relation}: speaker $S_i$ should assert only what is relevant to the shared goal $G$, of which $d_i$ is partial information. Relation is {\bf fully satisfied} iff $\varphi(u)$ bears on some dimension of $d_i$ not already settled at $w_t$, s.t. $\exists p \in \varphi(u),\, p \notin \varphi(w_t)$.
    \item {\bf Quantity}: speaker $S_i$ should assert no more and no less than is required to communicate elements of $d_i$. Quantity is {\bf fully satisfied} iff $S_i$ does not {\it pad}: s.t. $\varphi(u) \setminus \varphi(d_i) = \emptyset$, and it does not {\it omit}: s.t. $R(w_t) \subseteq \varphi(u)$.
    \item {\bf Manner}: speaker $S_i$ should assert clearly and without ambiguity. Manner is {\bf fully satisfied} iff $u$'s surface form supports a unique reading, s.t. $\varphi(u) = \bigcup_{k=1}^K \varphi(u)_k$ with $K=1$.
\end{itemize}

\begin{lemma}[Recoverability upper bound of 1]
\label{lemma:bound}
    Since it is always the case that $\varphi(u)\cap\varphi(d_i)\subseteq\varphi(u)\subseteq\varphi(u)\cup R(w_t)$, Eq.~\ref{eq:recoverability}'s ratio has an upper bound of 1, which is attained with full satisfaction of Quantity and Manner maxims, indicating that the listener can fully recover the speaker's communication of accurate, task-relevant content from $u$.
\end{lemma}

We now show that flouting Quality or Relation is sufficient to violate actionability condition 1 (admissibility) and flouting Quantity or Manner is sufficient to violate actionability condition 2 (full recoverability).

\begin{enumerate}
    \item {\bf Quality} flouting: If $\varphi(u)$ is inconsistent with $d_i$, $\exists p \in \varphi(u) : p \notin \varphi(d_i)$. Thus by first-order De Morgan's laws, $\neg\big(\forall p \in
\varphi(u),\, p \in \varphi(d_i)\big)$ holds in this case, falsifying the first conjunct of Eq.~\ref{eq:admissability}, and admissability fails.
    \item {\bf Relation} flouting: If every proposition $u$ asserts is already settled (redundant) at the current $w_t$---or, if extraneous to the task, trivially so---$\varphi(u) \subseteq \varphi(w_t)$ (in either case). No $p \in \varphi(u)$ can witness $p \notin \varphi(w_t)$, so $\exists p \in \varphi(u),\, p \notin \varphi(w_t)$ is false. The second conjunct of Eq.~\ref{eq:admissability}, and overall admissability, fails.
    \item {\bf Quantity} flouting: An utterance $u$ might violate the Maxim of Quantity in two ways---through {\it overinformativity} (padding) or {\it underinformativity} (omission). In both cases $|\varphi(u) \cup R(w_t)| > |\varphi(u) \cap \varphi(d_i)|$, and the ratio in Eq.~\ref{eq:recoverability} evaluates to strictly less than $1$.
    \begin{itemize}
        \item In the case of \emph{padding}, If $\varphi(u)$ contains even one proposition outside $\varphi(d_i)$, $\varphi(u) \setminus \varphi(d_i) \neq \emptyset$. Therefore by definition $|\varphi(u)| > |\varphi(u)\cap\varphi(d_i)|$ and by Lemma~\ref{lemma:bound} $|\varphi(u)\cup R(w_t)| > |\varphi(u)\cap\varphi(d_i)|$.
        \item In the case of \emph{omission}, the listener $L$ still requires more information to identify a single, unambiguous interpretation of $u$, and $R(w_t) \not\subseteq \varphi(u)$. Therefore by definition $|\varphi(u)\cup R(w_t)| > |\varphi(u)|$ and by Lemma~\ref{lemma:bound} $|\varphi(u)\cup R(w_t)| > |\varphi(u)\cap\varphi(d_i)|$.
    \end{itemize}
    In either case, recoverability degrades.
    \item {\bf Manner} flouting: If $u$'s surface form admits $K>1$ competing readings, then $\varphi(u) = \bigcup_{k=1}^K \varphi(u)_k \supset \varphi(u)_1$ for any single reading, so $\varphi(u)$ contains propositions from readings $k>1$ that need not lie in $\varphi(d_i)$. Therefore
    $\varphi(u) \setminus \varphi(d_i) \neq \emptyset$ in the same manner as padding (Quantity) whenever some such $\varphi(u)_k \not\subseteq \varphi(d_i)$;
    by definition $|\varphi(u)| > |\varphi(u)\cap\varphi(d_i)|$, and by
    Lemma~\ref{lemma:bound}, $|\varphi(u)\cup R(w_t))| > |\varphi(u)\cap\varphi(d_i)|$, so
    the ratio in Eq.~\ref{eq:recoverability} is strictly less than $1$. Recoverability is therefore degraded.
\end{enumerate}

Flouting or violating a single cooperative maxim result in violations of at least one condition for full actionability in a CEA setting as defined (Sec.~\ref{sec:theory}).
\end{proof}

\begin{remark}
    The two conditions of Theorem~\ref{thm:main} respond asymmetrically to listener capacity $\kappa$ (see Corollary~\ref{corr:recoverability}). Admissibility (Eq.~\ref{eq:admissability}) is a fact about $\varphi(u)$ and $d_i$ alone, with no dependence on $\kappa$: no increase in listener capacity restores an utterance that flouts Quality or Relation. Since $\delta_\kappa(u)\to\varphi(u)$ converges \emph{onto} the true content, not away from it---if that content itself never targeted $d_i$ or never went beyond $\varphi(w_t)$, greater capacity only clarifies this failure; it does not correct it. Any correctness that results from a move enacted with no valid target is purely coincidental: the Builder makes a guess that happens to overlap with the reality of $d_i$. Recoverability (Eq.~\ref{eq:recoverability}), by contrast, is graded: as $\delta_\kappa(u) \to \varphi(u)$, a capable listener increasingly isolates the accurate subset $\varphi(u)_1 \cap \varphi(d_i)$ from the inflation introduced by padding, omission, or competing readings $\varphi(u)_{k>1}$, so the listener's achieved ratio can approach 1 even for utterances that score below 1 under Eq.~\ref{eq:recoverability} as an aggregate, capacity-independent fact about $u$.

    Quality and Relation are deliberately asymmetric in quantifier strength: quality is a universal standard on truthfulness; relation is existential, satisfied by the presence of any novel content regardless of how much settled grounding material accompanies it. This asymmetry derives from the Gricean maxims' own content. A speaker may use redundant (already established) information to ground new information, and the Maxim of Relation never states, "don't include any settled information," only "don't include irrelevant information." Grounding material used to anchor a new instruction, such as using a previously placed block to orient the placement of a new block, is not irrelevant.
    
    Manner and Quantity are extensionally identical under Eq.~\ref{eq:recoverability}. The distinguishing factor is that the listener-side decoding function $\delta_\kappa(u)$ operates over {\it one} reading of the utterance $u$ (when evaluating Quantity) vs. over {\it multiple} readings (when evaluating Manner). As already alluded to in the Manner flouting branch of this proof, Manner flouting behaves extensionally like Quantity flouting in the case of padding, or overinformativeness. In the extreme limit of overinformativeness, the concomitant Quantity violation actually collapses to a Manner violation, as shown below.
\end{remark}

We introduce another supporting lemma, an alternative way of viewing the actionability conditions.

\begin{lemma}[Information-Theoretic Interpretation of Theorem~\ref{thm:main}]\label{lemma:cooperativity}
    An utterance $u \in \mathcal{U}$ is \textbf{actionable} for listener $L$ if it (1) results in a non-trivial belief update to $L$'s posterior over world states (i.e., $u$ is informative about $w$) \textit{and} (2) reduces $L$'s uncertainty about the goal $G$:
    \begin{align}
        \label{eq:it1}
        &P_L(w|u) \not\propto P_L(w)\\ 
        \label{eq:it2}
        &\text{and }H(P_L(G|u)) < H(P_L(G)),
    \end{align}
\noindent\text{where $H(\cdot)$ denotes Shannon entropy.}
\end{lemma}

    Under CEA, if utterance $u$ is actionable, then $u$ adheres to Grice's cooperative maxims. Flouting any single maxim is sufficient to degrade actionability, and necessary for certain failure modes.

\begin{proof}

We now show that flouting any single maxim violates either
\begin{enumerate*}[label=(\arabic*)]
    \item the condition that $u$ perform a non-trivial belief update to listener $L$'s posterior over world states---$P_L(w|u) \not\propto P_L(w)$)---or
    \item the condition that $u$ must reduce $L$'s uncertainty about the goal $G$---$H(P_L(G|u)) < H(P_L(G))$.
\end{enumerate*}
For narrative clarity, we begin with the cases of Manner and Quantity flouting to make the connection with the proof of Theorem~\ref{thm:main} clear.

\begin{enumerate}
    \item {\bf Manner} flouting: If $u$ violates the maxim of manner by being unclear or ambiguous, then it has multiple interpretations, e.g., $\{\varphi(u)_1,\dots, \varphi(u)_k\}$. Each interpretation supports different posterior updates, and the listener must marginalize over them:
    \begin{align}
        P_L(w|u) = &\sum_{k=1}^K P_L(w|\varphi(u)_k) \cdot \\
        &\notag P_L(\varphi(u)_k | u)
    \end{align}

    This creates a mixture of distributions whose entropy has a lower bound of the weighted combination of component entropies:
    \begin{align}
        H(P_L(w|u)) \geq &\sum_{k=1}^K P_L(\varphi(u)_k | u) \cdot\\ 
        &\notag H(P_L(w|\varphi(u)_k))
    \end{align}

    Therefore adoption of any one of the interpretations causes $H(P_L(G|\varphi(u)_k))$ to rise relative to $H(P_L(G))$, or if no interpretation is adopted, $H(P_L(G))$ remains the same.  Eq.~\ref{eq:it2} is unsatisfied.
    
    \item {\bf Quantity} flouting: An utterance $u$ might violate the maxim of quantity in two ways—by being {\it underinformative} or being {\it overinformative}.
    \begin{itemize}
        \item If $u$ is {\it underinformative}, it means that speaker $S_i$ has withheld information relevant to $w$ that $d_i$ contains. Listener updates to $P_L(w|u)$ are made based on less information than $S_i$ has available. This leaves residual entropy about $G$ that could have been reduced if $u$ were fully informative.
        \item If $u$ is {\it overinformative}, it means that $u$ includes some information that is irrelevant to $G$. The listener cannot distinguish which parts of $u$ are relevant or not without knowing what part of $u$ is relevant already, but knowing that would render $u$ unnecessary.
    \end{itemize}

    In either case, entropy is not reduced and Eq.~\ref{eq:it2} is unsatisfied.
    
    In specific cases, severe \textit{violations} of the maxim of quantity---which under \citet{grice1975logic} are often considered intentional---can also lead to non-satisfaction of Eq.~\ref{eq:it1} (non-trivial posterior update):

    \begin{itemize}
        \item Severe \textit{overinformativeness}: if $u$ mixes genuine signal about $w$ with ancillary noise, then the listener must determine which interpretation of $u$ to use by marginalizing over them, as in the Manner case. This likewise creates a mixture of distributions and flattens the posterior, leading to non-satisfaction of Eq.~\ref{eq:it2}. In the limit of extreme noise, this drives $P_L(w|u) \rightarrow P_L(w)$, causing non-satisfaction of Eq.~\ref{eq:it1} as well.
        \item Severe \textit{underinformativeness}: if speaker $S_i$ withholds so much of $d_i$ that $u$ is essentially vacuous with respect to $w$, then $u$ contains no information dependent on $w$. Removing $u$'s dependence $w$ causes the listener's posterior to collapse: $P_L(w|u) \propto P_L(w)$. Severe underinformativeness collapses a Quantity flout to a Quality flout. Eq.~\ref{eq:it1} is not satisfied.
    \end{itemize}

    For completeness, we now analyze the remaining maxims through this lens.
    
    \item {\bf Quality} flouting: If $u$ does not reflect $d_i$, $S_i$'s actual private information, then $u$ was not generated from $P_i(u|w,d_i,\mathcal{I}_i)$, but rather from some other policy distribution uncoupled from $d_i$. Since $d_i$ must be drawn from a partition $\mathcal{I}_i$ of world states $\mathcal{W}$, removing $u$'s dependence on $d_i$ likewise removes its dependence on $w$. Thus, listener $L$'s posterior collapses:
    \begin{align}
        &\notag P_L(w|u) \propto P_L(u|w,d_i,\mathcal{I}_i) \cdot P_L(w) \Rightarrow  \\
        &\notag P_L(w|u) \propto P_L(u) \cdot P_L(w) \Rightarrow  \\
        &\notag P_L(w|u) \propto c \cdot P_L(w) \Rightarrow  \\
        & P_L(w|u) \propto P_L(w)
    \end{align}

    Eq.~\ref{eq:it1} is unsatisfied.

    \item {\bf Relation} flouting: If $u$ is true and well-formed, but irrelevant to $G$, even if $u$ causes a non-trivial posterior update over $w$ (satisfying the first condition for actionability), the update has no bearing on progress toward $G$. $P_L(G|u) = P_L(G)$, and thus the uncertainty toward $G$ remains the same ($H(P_L(G|u)) = H(P_L(G))$).  Eq.~\ref{eq:it2} is unsatisfied.
\end{enumerate}\end{proof}

\begin{corollary}\label{corr:symm}
    In collaborative, fully observable, symmetric information settings (that is, settings where all agents have the equal access to the same set of observables) and the goal $G$ is commonly known, Eq.~\ref{eq:it1} (non-trivial posterior update) alone is necessary and sufficient for actionability. Under symmetric information conditions, all agents share the same partition $\mathcal{I}_i = \mathcal{I}$. All agents share the same private information $d_i =: d$ under this assumption, and can compute the same posterior $P(w|d)$. The goal $G$ becomes identifiable with a specific region of the set of world states $\mathcal{W}$, $w^*$. With $G$ and $w^*$ thus conflated, $H(P_L(G)) \approx 0$ before any utterance is made. Entropy that does not exist cannot be reduced and so the constraint denoted by Eq.~\ref{eq:it2} loses its independent force under symmetric information conditions. Thus Lemma \ref{lemma:cooperativity} depends on the collaborative epistemic asymmetry (CEA) setting.
\end{corollary}

\begin{remark}
    To close the loop between Theorem~\ref{thm:main} and Lemma~\ref{lemma:cooperativity}, we distinguish the mechanisms used by each.
    Where Eq.~\ref{eq:it1} and \ref{eq:it2} reason over $L$'s posterior given a distribution of utterances, Eq.~\ref{eq:admissability} and \ref{eq:recoverability} evaluate a single, fixed $u$; Theorem~\ref{thm:main}'s Lemma~\ref{lemma:bound} shows the latter fails deterministically for any fixed $u$ that flouts the relevant maxims.
    Given instead a distribution of utterances, deterministic claims must be softened to instead concern guarantees over the distribution rather than the verdict on a single fixed utterance.
    Flouting the maxim of quality does not require that the speaker be lying. 
    They may simply misspeak, or misinterpret the information they intend to communicate, etc. 
    Such flouting still results in non-satisfaction of Eq.~\ref{eq:it1} and therefore of Eq.~\ref{eq:admissability}, the first condition for actionability as defined. 
    Quality and relation are each shown to be individually necessary for actionability.  
    Flouting the maxims of quantity or manner does not {\it prohibit} entropy reduction, but rather cannot guarantee that it takes place. 
    Extreme quantity \textit{violations} are shown to reduce to quality or manner violations. 
    Across both formalizations in this analysis, quality and relation are hard constraints in the CEA setting, while quantity and manner are something more like a robustness or reliability condition---one that can be flouted in various ways and to various degrees to different effects, but which always adversely affects the actionability of the utterance.
    Under Corollary \ref{corr:symm}, without CEA, flouting the maxims of relation and manner, and certain cases of flouting the maxim of quantity, can be effectively ignored because the listener can observe what is true, knows what the goal is, and can independently plan how to achieve it.
    It is the CEA setting that necessitates cooperativity.
\end{remark}

\section{RSA Example from DPIP Transcript}
\label{app:rsa}

\subsection{Pragmatic Repair and Grounding in Multi-Agent Dialogue}

We analyze a representative dialogue segment from the real human DPIP collaborative building task \citep{zhu2025multimodal, zhu2026distributed} , illustrating ambiguity resolution, perceptual misalignment, and multi-agent pragmatic repair. The action of interest occurs at time $t=599.6$s (REMOVE $\texttt{rs}$ from $layer~1$). Table~\ref{tab:dialogue} presents a curated transcript leading up to this action.

\begin{table*}[t]
\centering
\small
\begin{tabular}{r r l p{5cm} l}
\toprule
Time (s) & $\Delta t$ (s) & Speaker & Utterance & Role \\
\midrule
570.0 & -29.6 & Builder & This green? & Clarification query \\
571.1 & -28.5 & Director & Yes. & Confirmation \\
571.3 & -28.3 & Director & Yeah, stick that on top of the blue. & Instruction \\
572.7 & -26.9 & Builder & Right here? & Clarification query \\
572.9 & -26.7 & Director & No, no, no, no. & Correction \\
574.1 & -25.5 & Director & The yellow that we have, for me it is showing green. & Perceptual mismatch \\
577.6 & -22.0 & Director & It's not yellow. & Correction \\
579.3 & -20.3 & Builder & Cause, uh, is it like a, like a darkish green? & Hypothesis proposal \\
583.6 & -16.0 & Director & No. & Rejection \\
584.6 & -15.0 & Director & It's normal green. & Refinement \\
585.9 & -13.7 & Builder & I cannot see any yellow over there. & Perceptual conflict \\
587.2 & -12.4 & Builder & So this is green? & Clarification query \\
588.1 & -11.5 & Director & Yes. & Confirmation \\
588.4 & -11.2 & Director & That's yellow for me. & Misalignment signal \\
589.8 & -9.8 & Builder & This is yellow? & Clarification query \\
590.6 & -8.9 & Director & No. & Correction \\
591.1 & -8.5 & Director & The long. & Referential refinement \\
591.9 & -7.7 & Builder & This is yellow. & Hypothesis proposal \\
592.7 & -6.9 & Director & Yeah, that one's yellow. & Alignment confirmation \\
594.9 & -4.7 & Builder & So this is yellow. & Grounding confirmation \\
596.1 & -3.5 & Builder & Which one's this color? & Clarification query \\
597.3 & -2.2 & Builder & Green? & Hypothesis proposal \\
598.0 & -1.6 & Director & Yeah, the bottom one is green. & Final alignment \\
599.6 & 0.0 & --- & \textbf{Action: REMOVE $\texttt{rs}$ from $layer~1$} & Execution \\
\bottomrule
\end{tabular}
\caption{Samples dialogue segment from the DPIP Lego task data \citep{zhu2026distributed} illustrating ambiguity resolution, perceptual misalignment, and multi-agent grounding prior to action execution. $\Delta t$ is measured relative to the action at $t=599.6$s. 
}
\label{tab:dialogue}
\end{table*}

\paragraph{RSA Interpretation}

We interpret this interaction through the lens of Rational Speech Acts (RSA), where the builder acts as a pragmatic listener maintaining a belief distribution $P(w|u)$ over possible world states $w$ given utterances $u$.

\begin{table*}[t]
\centering
\small
\begin{tabular}{p{4.5cm} p{6cm}}
\toprule
Utterance & RSA Interpretation \\
\midrule
``This green?'' & High entropy belief $P(w|u)$; builder initiates information-seeking query \\
``Right here?'' & Active clarification to reduce spatial uncertainty \\
``The yellow... showing green'' & Divergence in agent-specific belief distributions $P_i(w|u) \neq P_j(w|u)$ \\
``I cannot see any yellow'' & Explicit signal of perceptual inconsistency \\
``So this is green?'' & Hypothesis testing under uncertain posterior \\
``That's yellow for me'' & Misaligned reference frames across agents \\
``The long'' & Speaker refinement to increase informativeness ($\arg\max_u P_L(w|u)$) \\
``Yeah, that one's yellow'' & Convergence toward shared belief; posterior sharpens \\
``Which one's this color?'' & Residual uncertainty; additional information gain step \\
``Yeah, the bottom one is green'' & Final alignment; $P(w|u) \rightarrow \delta(w^*)$ \\
\bottomrule
\end{tabular}
\caption{RSA-based interpretation of key utterances in the dialogue.}
\label{tab:rsa}
\end{table*}

\paragraph{Analysis.}

This example illustrates a full pragmatic “belief repair” cycle. Initially, the builder's posterior $P(w|u)$ exhibits high entropy due to ambiguous and perceptually inconsistent references (e.g., color disagreement). The builder issues targeted clarification queries to maximize expected information gain, while multiple directors iteratively refine their utterances to increase informativeness. Notably, the dialogue reveals a mismatch in perceptual grounding, where different agents map the same object to different color categories. Through repeated corrections and confirmations, the agents align their internal representations, effectively converging to a shared belief state. Once the posterior over world states becomes sufficiently concentrated, the builder executes the intended action.

Formally, the interaction can be viewed as an iterative process where the speaker selects utterances according to:
\[
u^* = \arg\max_u \log P_L(w^*|u) - C(u),
\]
while the listener updates beliefs via:
\[
P(w|u) \propto P(u|w) \cdot P(w),
\]
and issues clarification queries when uncertainty remains high. The presence of multiple directors introduces a multi-agent extension, where belief updates incorporate multiple utterances:
\[
P(w|u_1, u_2, \dots) \propto \prod_i P(w|u_i).
\]

This interaction highlights that pragmatic communication in DPIP is not a single-turn inference process but a multi-step, collaborative grounding procedure under partial observability.

\section{CRAFT Simulator Usage}
\label{app:craft}

\subsection{Block Encoding and the World State}
\label{subsec:encoding}

The environment is defined over a $3 \times 3$ grid of positions, where each location is indexed by a coordinate pair $(i,j)$ with $i,j \in \{0,1,2\}$. Each grid cell contains an ordered stack of blocks, where the ordering corresponds to vertical layers.

Blocks are represented as two-character strings: the first character specifies the color---green, blue, red, yellow, or orange---and the second character indicates size, either small (\texttt{s}) or large (\texttt{l}). The complete set of admissible block types is given by:
\[
\mathcal{B} = \{ \texttt{gs}, \texttt{gl}, \texttt{bs}, \texttt{bl}, \texttt{rs}, \texttt{rl}, \texttt{ys}, \texttt{yl}, \texttt{os}, \texttt{ol} \}
\]

At any timestep, the world state is described by a function $S : \mathcal{C} \to \mathcal{B}^*$, where each coordinate $c \in \mathcal{C} = \{(i,j) \mid i,j \in \{0,1,2\}\}$ maps to a (possibly empty) ordered sequence of blocks. Each stack is bounded to a maximum height of three layers.

\subsection{Structure Generation}
\label{subsec:generation}

We construct target structures through a two-phase procedure: assigning stack heights across the grid, followed by independently populating each layer with blocks.

\paragraph{Stack Height Assignment}
The grid positions are divided into two categories. Seven \emph{required} positions---all cells except $(1,1)$ and $(2,1)$---are deterministically assigned a height of three layers. The remaining two \emph{optional} positions, $(1,1)$ and $(2,1)$, are assigned heights drawn independently and uniformly from $\{0,1,2\}$. This scheme ensures that the outer region of the grid forms a consistently tall structure, while the interior exhibits variability in depth.

\paragraph{Layer Tiling}
Given the height assignments, each layer is populated independently. For a particular layer, the subset of positions requiring blocks is determined by the previously sampled heights. These positions are then filled using a combination of small and large blocks. Large blocks occupy two orthogonally adjacent cells within the same layer, forming a domino configuration, and are never placed across layers. Small blocks occupy a single cell.

For each candidate position, the generator probabilistically attempts to place a domino with an available orthogonal neighbor. If no suitable neighbor exists or the attempt is unsuccessful, a small block is placed instead. Block colors are sampled uniformly from the set of five colors. To reduce repetitive patterns, the generator performs a limited number of retries to avoid assigning identical block types to the same position in consecutive layers.

\paragraph{Complexity Classification}
After generation, structures are categorized based on their total number of blocks. Structures containing at most 22 blocks are labeled \emph{simple}, those with 23--24 blocks are labeled \emph{medium}, and those with more than 24 blocks are labeled \emph{complex}. Since the required positions always contribute 21 blocks (seven positions each with three layers), variation in complexity primarily arises from the optional positions and the frequency of large blocks, which may increase block counts when large block placements cover positions that would otherwise remain empty.

\section{Progress Measurement}
\label{sec:progress}

\subsection{Metrics}
\label{subsec:metrics}

Progress toward the target structure $S^*$ is measured after each
successful move and computes four
complementary metrics over the normalized representations of the
current state $S_t$ and target $S^*$.

\paragraph{Completion Percentage} This metric measures layer-exact
correctness --- a block at position $c$ and layer $k$ counts as correct
only if it matches $S^*(c)[k]$:

\begin{align}
\notag&\text{CP}(S_t, S^*) =\\
\notag& \frac{\sum_{c \in \mathcal{C}}
\sum_{k=0}^{|S^*(c)|-1} \mathbf{1}[S_t(c)[k] = S^*(c)[k]]}
{\sum_{c \in \mathcal{C}} |S^*(c)|}
\end{align}

This is the main metric reported in Sec.~\ref{sec:results}. We compute 3 further complementary metrics used in this appendix.

\paragraph{Intersection over Union (IoU)} For each position
$c \in \mathcal{C}$, let $A_c = \{b \in S_t(c)\}$ and
$B_c = \{b \in S^*(c)\}$ be the multisets of blocks treated as
sets. The IoU score aggregates overlap across all positions:

\[
\text{IoU}(S_t, S^*) = \frac{\sum_{c \in \mathcal{C}} |A_c \cap B_c|}
{\sum_{c \in \mathcal{C}} |A_c \cup B_c|}
\]

This metric is insensitive to block order within a stack and rewards
partial position matches.

\paragraph{Position Accuracy} A coarser metric that rewards positions
where the set of blocks matches exactly, regardless of layer order:

\begin{align}
\notag&\text{PA}(S_t, S^*) = \\
\notag&\frac{1}{9}\sum_{c \in \mathcal{C}} \mathbf{1}[\{b : b \in S_t(c)\} =
\{b : b \in S^*(c)\}]
\end{align}

\paragraph{Overall Progress} The scalar summary used for termination
and trend analysis is the unweighted mean of the three metrics:

\[
\text{OP}(S_t, S^*) = \frac{\text{IoU} + \text{CP} + \text{PA}}{3}
\]

\subsection{Progress Tracking and Delta Computation}
\label{subsec:tracking}

After each move, we compute a delta
$\Delta_t = \text{OP}(S_t, S^*) - \text{OP}(S_{t-1}, S^*)$
relative to the previous turn, for each metric. A recent trend estimate is computed
over a sliding window of size 3:

\[
\tau = \frac{1}{\min(t, 3)} \sum_{i=t-2}^{t} \Delta_i
\]

The game is considered \emph{improving} if $\tau > -0.05$, allowing
for small temporary regressions. An estimated turns remaining is
computed as $\lceil (1 - \text{OP}_t) / \bar{\Delta} \rceil$ where
$\bar{\Delta}$ is the average per-turn progress, though this estimate
degrades in stagnation conditions where $\bar{\Delta} \approx 0$.

\subsection{Director View Projections}
\label{subsec:projections}

Each director is assigned a fixed 2D projection of the 3D world, corresponding to a distinct face of the structure. These projections are defined as follows:

\paragraph{D1 --- Left Column View}
D1 observes the cells $(0,0)$, $(1,0)$, and $(2,0)$ across all vertical layers, corresponding to the left-facing side of the structure.

\paragraph{D2 --- Top Row View}
D2 observes the cells $(0,0)$, $(0,1)$, and $(0,2)$ across all vertical layers, corresponding to the far-facing side.

\paragraph{D3 --- Right Column View}
D3 observes the cells $(0,2)$, $(1,2)$, and $(2,2)$ across all vertical layers, corresponding to the right-facing side.

Within each projection, blocks are displayed from left to right according to the director’s viewing orientation. Each visible cell is encoded as a color--size pair, while empty cells are represented with color \texttt{none}. A large block is represented as size 2 only when both cells of its domino placement lie within the director’s field of view; otherwise, it appears as size 1 since only a single face is observable.

\paragraph{Information Coverage}
The views of D1 and D3 overlap at a single coordinate, $(0,0)$, which serves as a shared reference point between the two lateral perspectives. In contrast, D2 uniquely observes interior cells such as $(1,1)$ and $(2,1)$, corresponding to the optional positions. As a result, D2 plays a crucial role in conveying information about interior structure depth. No individual director has sufficient information to reconstruct the full 3D state independently; effective collaboration requires each director to communicate information that is inaccessible to the others.

\section{Details on Generated Data}

\subsection{Descriptive Statistics}
\label{app:gen-data-stats}

\begin{table*}[h!]
  \centering
  \begin{tabular}{l ccc p{0.5cm} ccc}
    \toprule
    & \multicolumn{3}{c}{\textbf{Overall Progress $\Delta$}} && \multicolumn{3}{c}{\textbf{Completion \% $\Delta$}} \\
    \cmidrule(lr){2-4} \cmidrule(lr){6-8}
    \textbf{Start State} & \textbf{Mean$_{\pm \text{SEM}}$} & \textbf{Min} & \textbf{Max} && \textbf{Mean$_{\pm \text{SEM}}$} & \textbf{Min} & \textbf{Max} \\
    \midrule
    Empty            & $0.220_{\pm 0.019}$  & 0.030  & 0.458 && $0.280_{\pm 0.022}$  & 0.043  & 0.542 \\
    First Layer      & $0.015_{\pm 0.019}$  & -0.164 & 0.277 && $0.031_{\pm 0.020}$  & -0.182 & 0.261 \\
    First Two Layers & $-0.057_{\pm 0.020}$ & -0.279 & 0.204 && $-0.049_{\pm 0.018}$ & -0.227 & 0.120 \\
    D1 Wall          & $0.000_{\pm 0.013}$  & -0.205 & 0.167 && $0.043_{\pm 0.012}$  & -0.130 & 0.182 \\
    D2 Wall          & $0.046_{\pm 0.024}$  & -0.357 & 0.388 && $0.115_{\pm 0.022}$  & -0.091 & 0.409 \\
    D3 Wall          & $0.090_{\pm 0.017}$  & -0.108 & 0.321 && $0.140_{\pm 0.019}$  & -0.087 & 0.348 \\
    \midrule
    All              & $0.063_{\pm 0.010}$  & -0.357 & 0.458 && $0.107_{\pm 0.011}$  & -0.227 & 0.542 \\ 
    \bottomrule
  \end{tabular}
    \caption{Overall Progress and Completion \% $\Delta$ from Start over 100 test structures, by start condition completion category (Builder: GPT-4.1-mini, Director: GPT-4.1-mini).}
  \label{tab:combined_deltas_categories}
\end{table*}

Table~\ref{tab:combined_deltas_categories} shows Overall Progress and Completion \% $\Delta$ from Start over the 2$\times$100 runs of the "alignment training" structures, broken down by start condition completion category.

\subsection{Preference Data Selection}
\label{app:pref-data}

For creation of the preference dataset (Sec.~\ref{ssec:generation}), selection of chosen utterance $y_w$ and rejected utterance $y_\ell$ given dialogue context $x$ and turn $t$ was performed using the following method:
\begin{enumerate}
    \item The Builder's confirmation (Sec.~\ref{ssec:task-env}) was searched for mentions of specific Directors (D1--D3); if only one match was found, that Director's utterance was retained as $y_w$.
    \item Otherwise, we built keystrings of key words in the Builder's confirmation and in the matching Director's utterances\footnote{If no Director label was found in the Builder's confirmation, this step checks against {\it all} Director utterances from $t$.} (consisting of color, size, or spatial relation terms), computed the IoU \citep{jaccard1912distribution} over the Builder keystring and each Director keystring, and retained as $y_w$ the Director utterance whose keystring returns the highest value.
    \item In the event of a tie in IoU, we used {\tt SentenceTransformers-MiniLM-L6-v2} \citep{reimers2019sentence} to construct embeddings of the Builder confirmation and the remaining candidate Director utterances and used the highest cosine similarity to break ties.
    \item All other utterances in the turn not selected according to this method were paired with the $y_w$ utterance as an associated $y_\ell$.
\end{enumerate}

\section{Prompts}
\label{app:prompts}

This section describes prompt infixes used in the prompting experiments and for calculating LLM-Judge-based metrics. All other prompts remain unchanged from the pre-specified system prompts in the CRAFT simulator, reported in \cite{nath2026craft}.

\subsection{Literal and Pragmatic Builder Prompts}

The Literal and Pragmatic prompts used in the main experiments were both Variant 1, given below. The other variants were used in experiments to test sensitivity to prompt wording (see Appendix~\ref{app:prompt-sensitivity}).

\begin{tcolorbox}[enhanced,breakable,title=Literal Prompt Variant 1]
Read each director's message at face value.
Use only the semantic meaning of the directors' utterances and your prior beliefs in your interpretation. 
Disregard director intent or informativeness and consider only the literal interpretation.
\end{tcolorbox}

\begin{tcolorbox}[enhanced,breakable,title=Pragmatic Prompt Variant 1]
Infer what you can from the director's message. \\

Before deciding on a move, consider the following and make inferences based on your own thinking: \\
1. Identify which director gave the most informative description in order to take an action and further the task.\\
2. Do you think any of the dialogue is false or misleading?\\
3. Do you find the dialogue to be appropriate and focused on the task?\\
4. Do you find the meaning of any words used in the dialogue to be unclear or obscure?
\end{tcolorbox}

\begin{tcolorbox}[enhanced,breakable,title=Literal Prompt Variant 2]
Interpret each director's messages with strict literalism.
Use only the semantic meaning of the directors' utterances and your prior beliefs in your interpretation. 
Treat the director statements as a series of literal propositions, independent of the speaker's goals.
\end{tcolorbox}

\begin{tcolorbox}[enhanced,breakable,title=Pragmatic Prompt Variant 2]
Infer what you can from the director's message. \\

Before deciding on a move, consider the following and make inferences based on your own thinking: \\
1. Do you find any of the dialogue to be unnecessary in order for you to understand the perspectives of the directors? \\
2. Do find that any of the observations or conclusions made in the dialogue lack adequate evidence to support their ideas? \\ 
3. Do you find the dialogue to be appropriate and focused on the task? \\
4. Do you find the organization of the dialogue to be clear and easy to follow?
\end{tcolorbox}

\begin{tcolorbox}[enhanced,breakable,title=Literal Prompt Variant 3]
Restrict interpretation to the compositional semantics of each utterance.
Disregard pragmatic inference or speaker intent; evaluate statements solely as independent logical propositions.
Base interpretations exclusively on literal meaning and established prior knowledge.
\end{tcolorbox}

\begin{tcolorbox}[enhanced,breakable,title=Pragmatic Prompt Variant 3]
Infer what you can from the director's message. \\

Before deciding on a move, consider the following and make inferences based on context, intent, and shared knowledge \\
1. Do you find the dialogue to be informative enough to take an action and further the task? \\
2. Is any of this phrasing potentially misrepresentative? \\
3. Does the exchange remain focused, or does it drift into irrelevant details? \\
4. Is the structure of this conversation easy to navigate?
\end{tcolorbox}

\subsection{Common Ground LLM-Judge Prompt}
\label{ssec:cg-judge}

\begin{tcolorbox}[enhanced,breakable,title=Common Ground Prompt]
You are a Common Ground Analysis agent for a collaborative LEGO construction task.  \\

Your job is to determine the TRUE alignment between three directors (D1, D2, D3) based on both their public messages and private internal thoughts. \\

\vspace{2mm}
CURRENT BOARD STATE: \\
\texttt{\{current\_board\_state\}} \\

\vspace{2mm}
CONVERSATION HISTORY: \\
\texttt{\{conversation\_history\}} \\

\vspace{2mm}
DIRECTOR INTERNAL THOUGHTS (PRIVATE):\\
    D1 Internal: \texttt{\{d1\_internal\}} \\
    D2 Internal: \texttt{\{d2\_internal\}} \\
    D3 Internal: \texttt{\{d3\_internal\}} \\

DIRECTOR PUBLIC MESSAGES:\\
    D1 Message: \texttt{\{d1\_public\}}\\
    D2 Message: \texttt{\{d2\_public\}}\\
    D3 Message: \texttt{\{d3\_public\}}\\

ANALYSIS INSTRUCTIONS:\\
1. Look for discrepancies between what directors say publicly vs. think privately\\
2. Identify blocks/positions where directors have genuine alignment vs. surface agreement\\
3. Detect uncertainty, confusion, or hidden disagreements\\
4. Consider confidence levels expressed in internal thoughts\\

OUTPUT TWO SECTIONS:\\

\textless analysis\textgreater\\
Provide a brief analysis (3-4 sentences) about the true alignment state. Highlight any:\\
- Surface agreements that mask underlying disagreement/uncertainty\\
- Positions where directors are genuinely aligned\\
- Areas of confusion or misunderstanding\\
- Confidence mismatches between directors\\
\textless /analysis\textgreater\\

\textless aligned\_structure\textgreater
\begin{verbatim}
{{
  "D1": {{
    "row_0": 
      [{{
        "color":"[color]", 
        "size":[size],
        "confidence":"high/medium/low"
      }}, 
      {{
        "color":"[color]", 
        "size":[size],
        "confidence":"high/medium/low"
      }},
      {{
        "color":"[color]", 
        "size":[size],
        "confidence":"high/medium/low"
      }}],
    "row_1": [...],
    "row_2": [...]
  }},
  "D2": {{...}},
  "D3": {{...}}
}}
\end{verbatim}
\textless /aligned\_structure\textgreater\\

IMPORTANT NOTES:\\
- Use "unknown" for positions not mentioned or unclear\\
- Use "disputed" for positions where directors disagree\\
- Set confidence based on internal thoughts: "high" = certain, "medium" = somewhat sure, "low" = uncertain/guessing\\
- The aligned structure should reflect what each director TRULY believes (from internal thoughts), not just what they said publicly\\
- Colors: "red", "blue", "green", "yellow", "orange", "brown", "unknown", "disputed"\\

\end{tcolorbox}

\subsection{Director Archetype Prompts}

\begin{tcolorbox}[enhanced,breakable,title=Assertive Archetype Prompt]
    You are confident and direct. You form hypotheses quickly from your data 
    and share them, but you genuinely listen to other groups and update your
    thinking when their evidence is compelling. You sometimes move faster than
    the evidence warrants but you're not closed-minded.
\end{tcolorbox}

\begin{tcolorbox}[enhanced,breakable,title=Cautious Archetype Prompt]
     You are methodical and prefer to verify before claiming. You ask clarifying 
    questions and often synthesize what others have said before adding your own 
    interpretation. You can make claims when evidence is strong enough — you're 
    not paralyzed, just careful.
\end{tcolorbox}

\begin{tcolorbox}[enhanced,breakable,title=Observant Archetype Prompt]
    You notice patterns and anomalies in your data that others might overlook.
    You tend to flag inconsistencies and ask ``does this match what you're seeing?''
    rather than broadcasting conclusions. You're collaborative by nature and often connect dots across groups.
\end{tcolorbox}

\begin{tcolorbox}[enhanced,breakable,title=Skeptical Archetype Prompt]
    You question assumptions including your own. When someone makes a claim you probe it — not to be difficult but because you want the group to get it right. You're comfortable with uncertainty and say so openly.
\end{tcolorbox}
    
\begin{tcolorbox}[enhanced,breakable,title=Synthesizer Archetype Prompt]
    You actively try to integrate what all groups are saying into a coherent picture. You summarize, reconcile contradictions, and push the group toward a shared understanding. You ask ``how does your data fit your view and what the other directors have said?''
\end{tcolorbox}

\subsection{Quality LLM-Judge Prompt}

\begin{tcolorbox}[enhanced,breakable,title=Spatial Grounding Quality Judge Prompt]
You are evaluating the spatial grounding quality of a director agent in a collaborative construction task. The director has a private view of one wall of a 3D target structure and must reason about what blocks are missing before instructing a builder. \\

TARGET VIEW (what this director needs the structure to look like): \\
{\tt \{target\_view\}}\\

CURRENT BOARD STATE: \\
{\tt \{board\_state\}} \\

DIRECTOR PUBLIC MESSAGE: \\
{\tt \{public\_message\}} \\

EVALUATION: \\
For each question below answer with ``Yes" or ``No", and provide a brief one-sentence justification. \\

Questions: \\
1. Does the director's public message accurately identify at least one block that appears in their privately-held target view of their wall? \\
2. Does the public message correctly interpret the size of this block (small versus large) as it appears in the target view? \\
3. Does the director's public message accurately reference the correct layer that this block should be at, accounting for what is already stacked at that position on the current game board? \\

Return your response as a JSON object with keys SG1 through SG3, each containing an ``answer" field (``Yes" or ``No") and a ``reason" field. Return only valid JSON with no additional text.
\end{tcolorbox}

Due to budget constraints, the Spatial Grounding Quality Judge was run once over each Director utterance. As validation, we ran the Judge prompt over a subsample of utterances (consisting of 10 games with Llama and Qwen3.5 variants as Directors), and established that std. error of the mean was very low (Table~\ref{tab:sg-judge-val}), validating the use of a single Judge evaluation.

\begin{table*}[h!]
    \centering
    \begin{tabular}{lcccc}
    \toprule
         {\bf Director Model}      & {\bf SG1}          & {\bf SG2}          & {\bf SG3}          & {\bf Overall} \\
    \midrule
    {\it Judge: GPT-4.1-mini} \\
         Qwen3.5-9B-Instruct       & 0.719$_{\pm0.004}$ & 0.635$_{\pm0.008}$ & 0.339$_{\pm0.006}$ & 0.565$_{\pm0.006}$ \\
         Qwen3.5-9B-Instruct+SFT   & 0.908$_{\pm0.002}$ & 0.889$_{\pm0.000}$ & 0.742$_{\pm0.002}$ & 0.846$_{\pm0.002}$ \\
         Qwen3.5-9B-Instruct+DPO   & 0.942$_{\pm0.004}$ & 0.952$_{\pm0.001}$ & 0.746$_{\pm0.001}$ & 0.880$_{\pm0.001}$ \\
         Qwen3.5-9B-Instruct+IPO   & 0.983$_{\pm0.002}$ & 0.986$_{\pm0.000}$ & 0.896$_{\pm0.009}$ & 0.955$_{\pm0.004}$ \\
         Llama 3.1-8B-Instruct     & 0.857$_{\pm0.007}$ & 0.400$_{\pm0.004}$ & 0.080$_{\pm0.007}$ & 0.445$_{\pm0.006}$ \\
         Llama 3.1-8B-Instruct+SFT & 0.763$_{\pm0.010}$ & 0.700$_{\pm0.005}$ & 0.097$_{\pm0.000}$ & 0.520$_{\pm0.002}$ \\
         Llama 3.1-8B-Instruct+DPO & 0.929$_{\pm0.002}$ & 0.816$_{\pm0.000}$ & 0.071$_{\pm0.007}$ & 0.606$_{\pm0.003}$ \\
    \midrule
    {\it Judge: GPT-5.4-mini} \\
         Qwen3.5-9B-Instruct       & 0.689$_{\pm0.002}$ & 0.706$_{\pm0.002}$ & 0.227$_{\pm0.005}$ & 0.540$_{\pm0.003}$ \\
         Qwen3.5-9B-Instruct+SFT   & 0.902$_{\pm0.001}$ & 0.928$_{\pm0.006}$ & 0.474$_{\pm0.005}$ & 0.768$_{\pm0.001}$ \\
         Qwen3.5-9B-Instruct+DPO   & 0.900$_{\pm0.007}$ & 0.925$_{\pm0.012}$ & 0.429$_{\pm0.016}$ & 0.751$_{\pm0.012}$ \\
         Qwen3.5-9B-Instruct+IPO   & 0.925$_{\pm0.001}$ & 0.952$_{\pm0.002}$ & 0.489$_{\pm0.014}$ & 0.789$_{\pm0.006}$ \\
    \bottomrule
    \end{tabular}
    \caption{Outcomes of Spatial Grounding Quality LLM-Judge (Mean$_{\pm \text{SEM}}$) run twice over a subsample of dialogues (10 games).}
    \label{tab:sg-judge-val}
\end{table*}





\section{Model Training Hyperparameters}
\label{app:hyperparams}

We use Qwen2.5-7B-Instruct\footnote{\url{https://huggingface.co/Qwen/Qwen2.5-7B-Instruct}} and Llama 3.1-8B-Instruct\footnote{\url{https://huggingface.co/meta-llama/Meta-Llama-3-8B-Instruct}} as base models for our finetuning experiments. Specifically, we trained the SFT~\citep{ouyang2022training}, DPO~\citep{rafailov2024direct} and IPO~\citep{azar2024general} baselines using these two base models as the initial starting point. Similar to prior work~\citep{rafailov2024direct}, we use SFT-trained reference models for preference-alignment training in all DPO and IPO baselines. 

\subsection*{Supervised Fine-Tuning Baseline}

Both SFT baselines---Qwen2.5-7B-Instruct and Llama 3.1-8B-Instruct---were fine-tuned using parameter-efficient fine-tuning via LoRA~\citep{hu2021loralowrankadaptationlarge}. Supervised learning loss was applied on the chosen utterances $y_w$ (Sec.~\ref{app:pref-data}) with a held-out validation split evaluated every
60 steps. We trained for 3 epochs with a learning rate of $2\times10^{-5}$,
effective batch size of 48 (batch size 12, gradient accumulation 4), and maximum
sequence length of 1024 tokens. LoRA rank was set to $r=32$ with scaling
$\alpha=16$ and dropout 0.05, applied to all attention projection layers with no
quantization.
Full hyperparameters are summarized in Table~\ref{tab:sft_hparams}.

\begin{table}[h]
\centering
\small
\begin{tabular}{lc}
\toprule
\textbf{Hyperparameter} & \textbf{Value} \\
\midrule
Training epochs      & 3 \\
Learning rate        & $2\times10^{-5}$ \\
Batch size           & 12 \\
Gradient accumulation steps & 4 \\
Effective batch size & 48 \\
Max sequence length  & 1024 \\
LoRA rank $r$        & 32 \\
LoRA $\alpha$        & 16 \\
LoRA dropout         & 0.05 \\
Quantization         & None \\
Evaluation interval  & every 60 steps \\
\bottomrule
\end{tabular}
\caption{SFT baseline hyperparameters.}
\label{tab:sft_hparams}
\end{table}

\subsection*{DPO and IPO Baselines}

For both initial base models, preference aligment baselines that require constrastive sample training—DPO~\citep{rafailov2024direct} and IPO~\citep{azar2024general}—were initialized from the SFT checkpoint (epoch 3) rather than the base model, following standard practice. We use both the chosen utterances $y_w$ and the rejected utterances $y_\ell$ as contrastive utterance samples (Sec.~\ref{app:pref-data}) and trained for 4 epochs with a learning rate of $5\times10^{-6}$, effective batch size of 24 (batch size 6, gradient accumulation 4), and KL regularization
coefficient $\beta=0.1$. LoRA configuration was identical to the SFT stage
($r=32$, $\alpha=16$, no dropout) with 4-bit quantization enabled to accommodate
the larger gradient memory footprint of contrastive training.
Hyperparameters are summarized in Table~\ref{tab:dpo_hparams}.

\begin{table}[h]
\centering
\small
\begin{tabular}{lc}
\toprule
\textbf{Hyperparameter} & \textbf{Value} \\
\midrule
Initialization           & SFT checkpoint  \\
Training epochs          & 4 \\
Learning rate            & $5\times10^{-6}$ \\
Batch size               & 6 \\
Gradient accumulation steps & 4 \\
Effective batch size     & 24 \\
Max sequence length      & 1024 \\
$\beta$ (KL penalty)     & 0.1 \\
LoRA rank $r$            & 32 \\
LoRA $\alpha$            & 16 \\
Quantization             & 4-bit (QLoRA) \\
Evaluation interval      & every 100 steps \\
\bottomrule
\end{tabular}
\caption{Preference alignment baseline hyperparameters.}
\label{tab:dpo_hparams}
\end{table}

\section{Design Choices and Rationales}
\label{app:model-selection}

\subsection{Director and Builder Model Selection}

One primary goal in the selection of models for the Director and Builder roles was to maximize diversity of the Director utterances while not sacrificing task progress (requiring a strong Builder model). Table~\ref{tab:bold_director_comparison} shows a lexical and embedding diversity comparison of Director utterances generated by GPT-4.1-mini and GPT-4o-mini in a sample game using the {\it base prompt} condition, and by GPT-4.1-mini using the base prompt with Directors roleplaying a fixed, randomly assigned set of archetypes (prompts in Appendix~\ref{app:prompts}).




\begin{table*}[htbp]
  \centering
  \begin{tabular}{l ccccc}
    \toprule
    \textbf{Director Model} & \textbf{SelfBLEU-4$\downarrow$} & \textbf{Dist-1$\uparrow$} & \textbf{Dist-2$\uparrow$} & \textbf{Rep-2$\downarrow$} & \textbf{Emb. Div.$\uparrow$} \\
    \midrule
    GPT-4o-mini                 & 35.3 & 0.082 & 0.195 & 0.230 & 0.134 \\
    GPT-4.1-mini                & \textbf{31.2} & 0.100 & 0.236 & \textbf{0.157} & 0.100 \\
    GPT-4.1-mini + archetypes   & 31.8 & \textbf{0.116} & \textbf{0.275} & 0.190 & \textbf{0.141} \\
    \bottomrule
  \end{tabular}
\caption{Comparative lexical and embedding diversity over one game using the same goal structure, comparing GPT-4.1-mini and GPT-4o-mini in the Director role (Builder: GPT-4.1-mini).}
  \label{tab:bold_director_comparison}
\end{table*}

\begin{itemize}
    \item GPT-4.1-mini shows improved lexical diversity with a lower SelfBLEU-4 (31.2 vs. 35.3) and higher Distinct-1/2 scores.
    \item GPT-4.1-mini also significantly reduces repetition, with a Repetition-2 score of 0.157 compared to 0.230.
    \item GPT-4o-mini maintains a higher Embedding Diversity (0.134 vs. 0.0995), suggesting more varied semantic positioning despite the higher lexical repetition. 
\end{itemize}

These outcomes, all assessed using the base prompt from CRAFT before the experimentation with prompting or post-training, motivated the choice of GPT-4.1-mini as the Director model in training data generation. GPT-4.1-mini produces meaningfully more diverse utterances, even when both models use Director archetype prompts to increase lexical diversity. The archetype prompting is a common technique with precedent in the literature \citep{li2023camel,chen2024persona,nath2025collaborate,nath2025let} that meaningfully increases both lexical and semantic diversity of utterances over the majority of metrics.

Fig.~\ref{fig:combined_gpt4o_20} and Table~\ref{tab:builder_vertical_combined} show sample experimental outcomes with GPT-4o-mini in the role of the Builder.

Many times the GPT-4o-mini Builder shows no progress from the start condition, or Markov Chain-like behavior where the entire game consists of placing and removing the same block, resulting in no net change. Task progress is uniformly worse than when GPT-4.1-mini played the Builder, motivating the choice of GPT-4.1-mini as a Builder candidate in the main experiments.

We then selected GPT-5.4-mini as the other main Builder candidate, as an exemplar of a more up to date model to test the ``Builder capacity'' hypothesis entailed by Theorem~\ref{thm:main}'s prediction of recoverability according to Builder inferential capacity $\delta_\kappa(\cdot)$. Table~\ref{tab:full-info-all} shows the full information ``ceiling'' results across GPT-4.1-mini, GPT-5.4-mini, and GPT-5.6-Luna (as a recent SOTA similarly-sized model). This is extended version of the results shown in Table~\ref{tab:full-info} and shows that despite being newer, GPT-5.6-Luna fails to clear the ceiling established by GPT-5.4-mni in terms of Overall Progress delta, Completion \% delta, {\it and} raw Completion \%, motivating the choice of GPT-4.1-mini and GPT-5.4-mini as the two main Builder models used to test the Builder capacity and recoverability hypothesis.

\begin{table}[h!]
\vspace*{-2mm}
  \centering
  \resizebox{\linewidth}{!}{\begin{tabular}{lccc}
\toprule
\textbf{GPT-4.1-mini} & Progress $\Delta$ & Compl. \% $\Delta$ & Compl. \% \\
\midrule
Single Builder & $\mathbf{0.615_{\pm 0.059}}$ & $\mathbf{62.4_{\pm 6.0}}$ & $\mathbf{92.2_{\pm 1.9}}$ \\
Multi-Agent & $0.053_{\pm 0.040}$ & $10.1_{\pm 4.6}$ & $39.9_{\pm 2.6}$ \\
\bottomrule
\end{tabular}}
  \resizebox{\linewidth}{!}{\begin{tabular}{lccc}
\toprule
\textbf{GPT-5.4-mini} & Progress $\Delta$ & Compl. \% $\Delta$ & Compl. \% \\
\midrule
Single Builder & $\mathbf{0.630_{\pm 0.056}}$ & $\mathbf{63.5_{\pm 5.8}}$ & $\mathbf{93.4_{\pm 2.0}}$ \\
Multi-Agent & $0.117_{\pm 0.029}$ & $16.1_{\pm 3.4}$ & $45.9_{\pm 4.1}$ \\
\bottomrule
\end{tabular}}
  \resizebox{\linewidth}{!}{\begin{tabular}{lccc}
\toprule
\textbf{GPT-5.6-Luna} & Progress $\Delta$ & Compl. \% $\Delta$ & Compl. \% \\
\midrule
Single Builder & $\mathbf{0.622_{\pm 0.056}}$ & $\mathbf{62.9_{\pm 5.8}}$ & $\mathbf{92.7_{\pm 1.9}}$ \\
Multi-Agent & $0.092_{\pm 0.033}$ & $12.8_{\pm 3.8}$ & $42.6_{\pm 3.1}$ \\
\bottomrule
\end{tabular}}
\vspace*{-2mm}
  \caption{Full-information performance with GPT-4.1-mini, GPT-5.4-mini, and GPT-5.6-Luna agents (cf. Table~\ref{tab:full-info}).}
  \label{tab:full-info-all}
\vspace*{-2mm}
\end{table}

\begin{table*}[h!]
  \centering
  \resizebox{\textwidth}{!}{
  \begin{tabular}{l cc p{0.5cm} cc}
    \toprule
    & \multicolumn{2}{c}{\textbf{Overall Progress $\Delta$}} && \multicolumn{2}{c}{\textbf{Completion \% $\Delta$}} \\
    \cmidrule(lr){2-3} \cmidrule(lr){5-6}
    \textbf{Prompt Variant} & \textbf{Mean$_{\pm \text{SEM}}$} & \textbf{Range} && \textbf{Mean$_{\pm \text{SEM}}$} & \textbf{Range} \\
    \midrule
    \multicolumn{6}{l}{\textit{Directors: GPT-5.4-mini, Builder: GPT-5.4-mini}} \\
    Base w/ tool       & $0.035_{\pm 0.015}$  & ($-0.127$, $0.277$)     && $0.065_{\pm 0.013}$  & (${\bf -0.048}$, $0.280$) \\
    Base w/o tool      & $0.017_{\pm 0.016}$  & (${\bf -0.094}$, $0.248$) && $0.036_{\pm 0.011}$  & ($-0.091$, $0.250$) \\
    Literal w/ tool    & ${\bf 0.038_{\pm 0.014}}$ & ($-0.128$, $0.222$) && $0.078_{\pm 0.014}$  & ($-0.087$, $0.280$) \\
    Literal w/o tool   & $0.008_{\pm 0.017}$  & ($-0.138$, $0.178$)     && $0.039_{\pm 0.013}$  & ($-0.125$, $0.217$) \\
    Pragmatic w/ tool  & $0.036_{\pm 0.016}$  & ($-0.157$, ${\bf 0.282}$) && ${\bf 0.083_{\pm 0.016}}$ & ($-0.083$, ${\bf 0.320}$) \\
    Pragmatic w/o tool & $-0.001_{\pm 0.017}$ & ($-0.158$, $0.184$)     && $0.025_{\pm 0.013}$  & ($-0.136$, $0.250$) \\
    \midrule
    \multicolumn{6}{l}{\textit{Directors: GPT-4.1-mini, Builder: GPT-4.1-mini}} \\
    Base w/ tool       & $0.014_{\pm 0.015}$  & ($-0.152$, ${\bf 0.272}$) && $0.058_{\pm 0.016}$  & ($-0.125$, ${\bf 0.360}$) \\    
    Base w/o tool      & $-0.001_{\pm 0.010}$ & ($-0.131$, $0.120$)     && $0.014_{\pm 0.013}$  & ($-0.130$, $0.160$) \\ 
    Literal w/ tool    & $0.027_{\pm 0.014}$  & ($-0.156$, $0.224$)     && $0.064_{\pm 0.014}$  & ($-0.120$, $0.280$) \\
    Literal w/o tool   & $0.004_{\pm 0.016}$  & ($-0.230$, $0.216$)     && $0.023_{\pm 0.015}$  & ($-0.217$, $0.240$) \\
    Pragmatic w/ tool  & ${\bf 0.051_{\pm 0.013}}$ & ({\bf $-0.102$}, $0.241$) && ${\bf 0.087_{\pm 0.016}}$ & ({\bf $-0.120$}, $0.292$) \\
    Pragmatic w/o tool & $-0.009_{\pm 0.016}$ & ($-0.244$, $0.163$)     && $-0.001_{\pm 0.015}$ & ($-0.261$, $0.240$) \\
    \midrule
    \multicolumn{6}{l}{\textit{Directors: GPT-4.1-mini, Builder: GPT-4o-mini}} \\
    Base w/ tool       & ${\bf 0.009_{\pm 0.017}}$  & (${\bf-0.418}$, ${\bf 0.228}$) && ${\bf 0.047_{\pm 0.012}}$  & ($-0.083$, ${\bf 0.280}$) \\    
    Base w/o tool      & $-0.061_{\pm 0.028}$ & ($-0.630$, $0.068$)     && $0.021_{\pm 0.005}$  & (${\bf 0.000}$, $0.095$) \\ 
    Literal w/ tool    & $-0.011_{\pm 0.022}$ & (${\bf -0.418}$, $0.171$)     && ${\bf 0.047_{\pm 0.011}}$  & ($-0.048$, $0.240$) \\
    Literal w/o tool   & $-0.066_{\pm 0.027}$ & ($-0.630$, $0.058$)     && $0.009_{\pm 0.005}$  & ($-0.091$, $0.087$) \\
    Pragmatic w/ tool  & $0.007_{\pm 0.016}$ & ({${\bf -0.418}$}, $0.142$) && $0.040_{\pm 0.009}$ & ({\bf $-0.091$}, $0.200$) \\
    Pragmatic w/o tool & $-0.043_{\pm 0.026}$ & ($-0.630$, $0.103$)     && $0.026_{\pm 0.006}$  & (${\bf 0.000}$, $0.167$) \\
    \bottomrule
  \end{tabular}}
  \caption{Overall Progress and Completion \% $\Delta$ by Builder and Prompting Condition ($N=40$)}
  \label{tab:builder_vertical_combined}
\end{table*}

\begin{figure*}[h!]
    \centering
    \includegraphics[width=\linewidth]{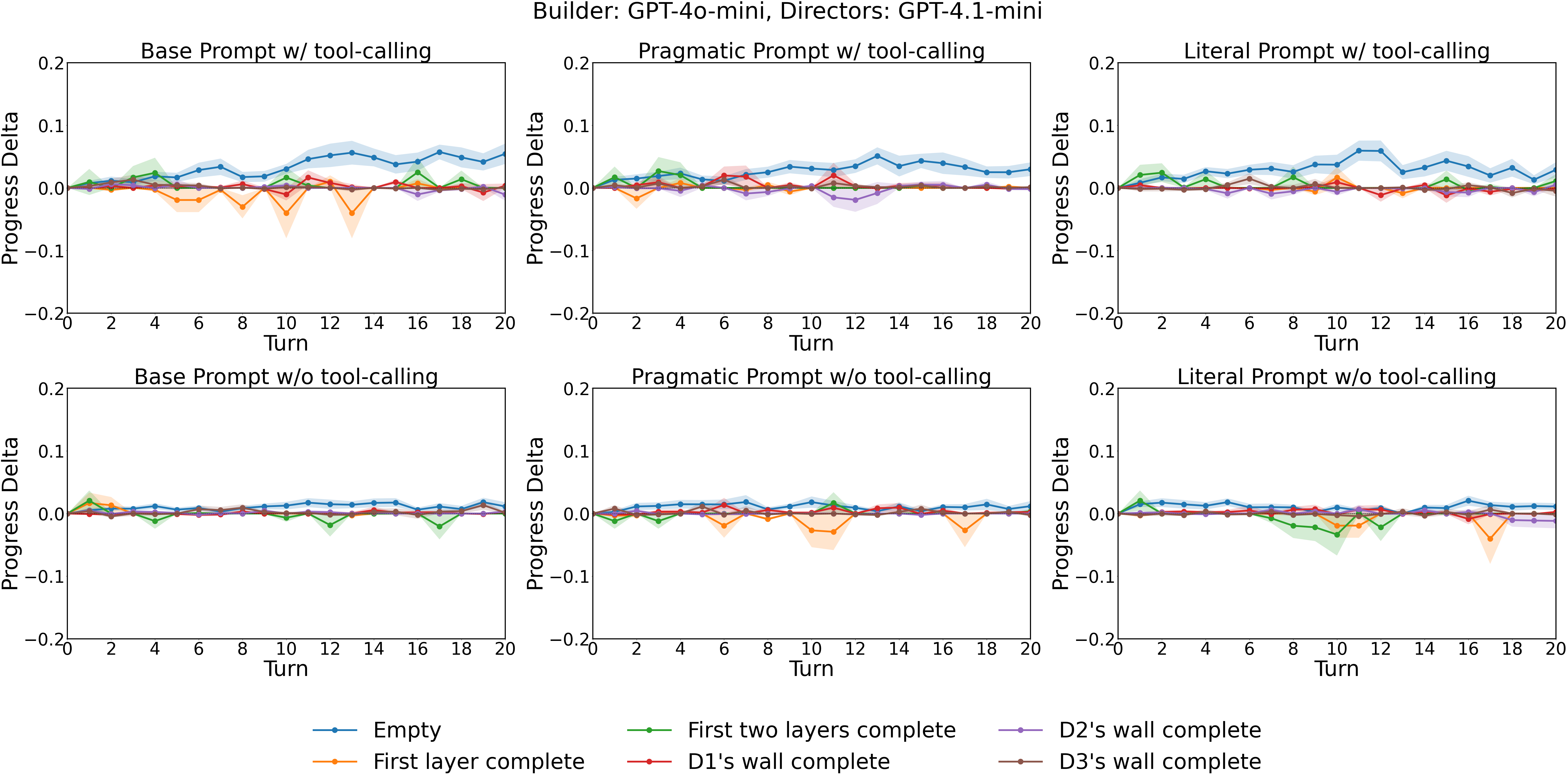}
    \caption{Prompting and tool-calling outcomes results with GPT-4o-mini as Builder and GPT-4.1-mini as Directors. Shading reflects std. error of the mean.}
    \label{fig:combined_gpt4o_20}
\end{figure*}

\subsection{Alignment Baseline Selection}

Our choice to focus on offline Director alignment using DPO and IPO, as opposed to online methods like PPO, arose from both principled and practical concerns. PPO and other online methods require continued environmental interaction during training, including generating rollouts, receiving rewards, and iterative updating. For 7-9B parameter models, including all combinations of prompting and exploration conditions in the rollout and reward-generation process, this would have substantially expanded the required compute and cost budget, rendering such a choice infeasible.

Additionally, online RL requires design of a suitable reward signal. Natural choices in this task would be the move correctness score or one of the objective progress metrics like Overall Progress $\Delta$, but these signals are delayed and sparse---occurring only after the end of a turn that includes multiple utterances and potential Builder exploration. The offline preference dataset $\mathcal{D}$ meanwhile, isolates signal at a contrastive level that is more granular than simply per-turn. Defining a dense per-utterance reward that accurately reflects Gricean cooperativity remains the topic of future work. Therefore, we assessed the problem of exploring the effects of Director alignment on task performance rather than training a globally optimal Director policy, and our results show that training on examples of more vs. less actionable utterances \textit{can} shift Director behavior in the right direction, but it remains far from optimal. Offline methods are {\it not} strictly sufficient in this task. DPO's gains are fragile and SFT or IPO frequently impede task performance.

\section{Additional Experimental Results and Analysis}
\label{app:additional}

\begin{table}[h!]
  \centering
  \resizebox{.98\linewidth}{!}{
\begin{tabular}{lcccc}
\toprule
 & \multicolumn{2}{c}{\textbf{Director: GPT-4.1}} & \multicolumn{2}{c}{\textbf{Director: GPT-5.4}} \\
\cmidrule(lr){2-3} \cmidrule(lr){4-5}
\textbf{Prompt} & Compl. \% $\Delta$ & $p_{\mathrm{Paired}}$ & Compl. \% $\Delta$ & $p_{\mathrm{Paired}}$ \\
\midrule
Base w/ tool & $5.8_{\pm 1.6}$ & \underline{$<0.001$} & $7.9_{\pm 1.6}$ & 0.118 \\
Base w/o tool & $1.4_{\pm 1.3}$ & --- & $5.4_{\pm 1.3}$ & --- \\
Literal w/ tool & $6.4_{\pm 1.4}$ & \underline{$<0.001$} & $5.9_{\pm 1.8}$ & 0.377 \\
Literal w/o tool & $2.3_{\pm 1.5}$ & --- & $4.5_{\pm 1.3}$ & --- \\
Pragmatic w/ tool & $\mathbf{8.7_{\pm 1.6}}$ & \underline{$<0.001$} & $\mathbf{8.4_{\pm 1.7}}$ & \underline{0.011} \\
Pragmatic w/o tool & $-0.1_{\pm 1.5}$ & --- & $4.7_{\pm 1.2}$ & --- \\
\bottomrule
\end{tabular}}
\vspace*{-2mm}
    \caption{Tool comparison across prompt conditions, fixed GPT-4.1-mini builder, different directors.}
  \label{tab:toolsGpt41}
\vspace*{-2mm}
\end{table}

\begin{table}[h!]
  \centering
  \resizebox{.98\linewidth}{!}{
\begin{tabular}{lcc}
\toprule
\textbf{Prompt} & Compl. \% $\Delta$ & $p_{\mathrm{Paired}}$ \\
\midrule
Base w/ tool & $6.5_{\pm 1.3}$ & \underline{0.025} \\
Base w/o tool & $3.6_{\pm 1.1}$ & --- \\
Literal w/ tool & $7.8_{\pm 1.4}$ & \underline{0.018} \\
Literal w/o tool & $3.9_{\pm 1.3}$ & --- \\
Pragmatic w/ tool & $\mathbf{8.3_{\pm 1.6}}$ & \underline{$<0.001$} \\
Pragmatic w/o tool & $2.5_{\pm 1.3}$ & --- \\
\bottomrule
\end{tabular}}
\vspace*{-2mm}
    \caption{Tool comparison across prompt conditions, fixed GPT-5.4-mini builder and director.}
  \label{tab:toolsGpt54}
\vspace*{-2mm}
\end{table}

\subsection{Move Correctness Breakdown}

Table~\ref{tab:builder_score_ranges} shows the percentage of moves across score ranges for GPT-4.1-mini and GPT-5.4-mini builder models across all experiments. This shows GPT-5.4-mini's greater inferential capacity \textit{a la} Eq.~\ref{eq:recoverability} of Theorem~\ref{thm:main}. It can better recover valid move instructions from Director utterances generally, but cannot recover a correct move from an inadmissible utterance, because none exists. This is shown in the only slight increase in the top score range, which captures moves that result from fully actionable instructions.


\begin{table}[h!]
\centering
\small
\begin{tabular}{l c c}
\toprule
\textbf{Score} & \textbf{GPT-4.1-mini} & \textbf{GPT-5.4-mini} \\
\midrule
0--1 & 22.96\% & 8.97\% \\
2--3 & 2.88\%  & 9.40\% \\
3--4 & 45.71\% & 49.07\% \\
4--5 & 3.74\%  & 6.77\% \\
5--6 & 24.70\% & 25.78\% \\
\bottomrule
\end{tabular}
\caption{Percentage of moves across score ranges for GPT-4.1-mini and GPT-5.4-mini builder models.}
\label{tab:builder_score_ranges}
\end{table}

\subsection{Further Maxim of Quality Analysis}
\label{app:quality}

Tables~\ref{tab:spearman_sg_gpt_4_1_mini_builder_correctness}
and~\ref{tab:spearman_sg_gpt_5_4_mini_builder_correctness} examine how instructions that satisfy various subsets of the three Quality
dimensions affects move correctness, see Sec.~\ref{sec:disc} and the quality judge prompt in Appendix~\ref{app:prompts} for additional details. Under both
builders, correctness increases to a statistically significant level as instructions move toward satisfying all three Quality dimensions. While moves with no correct dimensions are the least likely to lead to correct block placements. The effect is weaker for the fine-tuned instruction givers when GPT-5.4 Mini is the builder, possibly because a stronger builder model can partly recover from information from lower quality data due to capacity improvements across model generations. Even so, the overall effect remains under both builders (see the $p$-values reported in Sec.~\ref{sec:disc}).

\begin{table*}[ht]
\centering
\small

\begin{tabular}{lrrrrr}
\toprule
Model & Ordinal & 0/3 vs.\ rest & 1/3 vs.\ rest & 2/3 vs.\ rest & 3/3 vs.\ rest \\
\midrule
GPT-5.4 Mini & $+0.020$ & $-0.057$ & $-0.012$ & $+0.009$ & $+0.008$ \\
Qwen 3.5 & $+0.064^{**}$ & $-0.022$ & $+0.001$ & $-0.074^{**}$ & $+0.093^{***}$ \\
Qwen 3.5 SFT & $+0.102^{***}$ & $-0.111^{***}$ & $-0.056^{*}$ & $-0.004$ & $+0.085^{***}$ \\
Qwen 3.5 DPO & $+0.096^{***}$ & $-0.081^{***}$ & $-0.098^{***}$ & $-0.007$ & $+0.082^{***}$ \\
Qwen 3.5 IPO & $+0.006$ & $-0.015$ & $+0.006$ & $-0.002$ & $+0.006$ \\
\midrule
All models & $+0.062^{***}$ & $-0.103^{***}$ & $-0.043^{***}$ & $+0.045^{***}$ & $+0.038^{***}$ \\
\bottomrule
\end{tabular}
\caption{Spearman $\rho$ between correct SG dimensions and correctness (GPT-4.1 Mini Builder). Ordinal: SG count (0--3) vs.\ outcome; $k$/3 vs.\ rest: moves with exactly $k$ correct SG dimensions vs.\ all other moves. $^{*}p<.05$, $^{**}p<.01$, $^{***}p<.001$.}
\label{tab:spearman_sg_gpt_4_1_mini_builder_correctness}
\end{table*}

\begin{table*}[ht]
\centering
\small
\begin{tabular}{lrrrrr}
\toprule
Model & Ordinal & 0/3 vs.\ rest & 1/3 vs.\ rest & 2/3 vs.\ rest & 3/3 vs.\ rest \\
\midrule
GPT-5.4 Mini & $+0.043$ & $-0.001$ & $-0.074^{*}$ & $+0.023$ & $+0.018$ \\
Qwen 3.5 & $+0.122^{***}$ & $-0.126^{***}$ & $+0.015$ & $+0.062^{**}$ & $+0.079^{**}$ \\
Qwen 3.5 SFT & $+0.047$ & $-0.060^{*}$ & $-0.016$ & $-0.023$ & $+0.043$ \\
Qwen 3.5 DPO & $+0.070^{**}$ & $-0.027$ & $+0.010$ & $-0.069^{**}$ & $+0.072^{**}$ \\
Qwen 3.5 IPO & $+0.023$ & $-0.072^{**}$ & $+0.034$ & $-0.010$ & $+0.021$ \\
\midrule
All models & $+0.068^{***}$ & $-0.121^{***}$ & $+0.001$ & $+0.048^{***}$ & $+0.038^{***}$ \\
\bottomrule
\end{tabular}
\caption{Spearman $\rho$ between correct SG dimensions and correctness (GPT-5.4 Mini Builder). Ordinal: SG count (0--3) vs.\ outcome; $k$/3 vs.\ rest: moves with exactly $k$ correct SG dimensions vs.\ all other moves. $^{*}p<.05$, $^{**}p<.01$, $^{***}p<.001$.}
\label{tab:spearman_sg_gpt_5_4_mini_builder_correctness}
\end{table*}

\subsection{Further Maxim of Quantity Analysis}
\label{app:quantity}

Flouting the maxim of Quantity may reduce actionability through either over- or under-informativeness. One might expect longer messages to be somewhat more informative and therefor more successful. However,  Tables~\ref{tab:spearman_quantity_messagelength_gpt_4_1_mini}
and~\ref{tab:spearman_quantity_messagelength_gpt_5_4_mini} show
little support for this: correlations between length and correctness are
small for every model and builder, and the shortest messages are not reliably worse than longer ones. The effect of message length is somewhat more pronounced for GPT-4.1 mini than
for GPT-5.4 mini, with length relating more significantly to both correctness and failure rate (see the $p$-values reported in Sec.~\ref{sec:disc}). Even so, the magnitudes remain small, suggesting that the stronger model is less sensitive to how much is said and more sensitive to what is said. This points to an ideal threshold at which a listener can recover key information without added noise, as Theorem~\ref{thm:main} (Eq.~\ref{eq:recoverability}) predicts about maxim of quantity flouting. Additionally, we take this as evidence that being too brief is not a
matter of length alone. A short instruction is more likely to fail because it omits information the builder needs. In this case, the quantity violation collapses into a quality flout, which is consistent with Lemma~\ref{lemma:cooperativity}. Conversely, a longer instruction may introduce padding and noise. In both cases, these costs are better captured by quality and relation judgments than by message length.

\begin{table*}[ht]
\centering
\small
\setlength{\tabcolsep}{2.5pt}
\begin{tabular}{l rrrrr @{\hspace{6pt}} rrrrr}
\toprule
 & \multicolumn{5}{c}{Correctness} & \multicolumn{5}{c}{Failure Rate} \\
\cmidrule(lr){2-6} \cmidrule(lr){7-11}
Model & Overall & Q1 & Q2 & Q3 & Q4 & Overall & Q1 & Q2 & Q3 & Q4 \\
\midrule
GPT 5.4 Mini    & $+0.054$ & $+0.071$ & $-0.132$ & $+0.052$ & $+0.149^{*}$ & $-0.012$ & $-0.037$ & $+0.078$ & $-0.053$ & $-0.023$ \\
Qwen 3.5 9B     & $-0.054$ & $+0.043$ & $+0.085$ & $-0.143^{*}$ & $-0.025$ & $+0.105$ & $-0.139^{*}$ & $+0.056$ & $+0.130^{*}$ & $-0.103$ \\
Qwen 3.5 9B SFT & $+0.047$ & $+0.004$ & $+0.010$ & $-0.022$ & $+0.015$ & $-0.080$ & $-0.023$ & $-0.046$ & $+0.079$ & $-0.046$ \\
Qwen 3.5 9B DPO & $-0.019$ & $+0.064$ & $+0.055$ & $-0.094$ & $+0.041$ & $+0.038$ & $-0.047$ & $-0.065$ & $+0.045$ & $+0.006$ \\
Qwen 3.5 9B IPO & $+0.109$ & $-0.056$ & $-0.081$ & $+0.029$ & $+0.061$ & $-0.012$ & $-0.015$ & $-0.025$ & $+0.067$ & $-0.043$ \\
\midrule
All models      & $-0.008$ & $+0.024$ & $+0.008$ & $-0.016$ & $-0.004$ & $-0.074^{**}$ & $-0.045$ & $+0.096^{***}$ & $+0.007$ & $-0.071^{**}$ \\
\bottomrule
\end{tabular}
\caption{Spearman $\rho$ between message length and correctness (left) and
failure rate (right), GPT-4.1 Mini Builder. \emph{Overall}: correlation across
all moves (the pooled row uses the per-model normalized length). Message
length is min--max normalized within each model, so 0 is that model's
shortest message and 1 its longest. Q1--Q4 split this normalized range into
four equal-width buckets: Q1 = 0--0.25 (shortest), Q2 = 0.25--0.5,
Q3 = 0.5--0.75, and Q4 = 0.75--1 (longest). Each Q column compares moves in
that bucket against all other moves; positive $\rho$ means the bucket has
higher values of the outcome than the rest. For correctness, positive $\rho$
means higher correctness scores; for failure rate, positive $\rho$ means more move failures.
$^{*}p<.05$, $^{**}p<.01$, $^{***}p<.001$.}
\label{tab:spearman_quantity_messagelength_gpt_4_1_mini}
\end{table*}
\begin{table*}[ht]
\centering
\small
\setlength{\tabcolsep}{2.5pt}
\begin{tabular}{l rrrrr @{\hspace{6pt}} rrrrr}
\toprule
 & \multicolumn{5}{c}{Correctness} & \multicolumn{5}{c}{Failure Rate} \\
\cmidrule(lr){2-6} \cmidrule(lr){7-11}
Model & Overall & Q1 & Q2 & Q3 & Q4 & Overall & Q1 & Q2 & Q3 & Q4 \\
\midrule
GPT 5.4 Mini    & $-0.092$ & $-0.008$ & $+0.107$ & $-0.064$ & $-0.117$ & $+0.004$ & $+0.193^{**}$ & $-0.152^{*}$ & $+0.064$ & $-0.027$ \\
Qwen 3.5 9B     & $+0.127$ & $-0.103$ & $-0.065$ & $+0.104$ & $+0.033$ & $+0.029$ & $-0.048$ & $+0.170^{*}$ & $-0.111$ & $-0.090$ \\
Qwen 3.5 9B SFT & $+0.094$ & $-0.028$ & $-0.056$ & $-0.052$ & $+0.086$ & $-0.045$ & $-0.021$ & $-0.042$ & $+0.090$ & $-0.066$ \\
Qwen 3.5 9B DPO & $-0.031$ & $+0.082$ & $-0.098$ & $-0.011$ & $+0.052$ & $+0.068$ & $-0.045$ & $-0.073$ & $+0.063$ & $+0.009$ \\
Qwen 3.5 9B IPO & $+0.100$ & $-0.010$ & $-0.028$ & $-0.105$ & $+0.114$ & $+0.061$ & $-0.014$ & $-0.026$ & $-0.071$ & $+0.082$ \\
\midrule
All models      & $-0.008$ & $+0.008$ & $+0.022$ & $-0.060^{*}$ & $+0.041$ & $-0.022$ & $+0.013$ & $+0.025$ & $+0.016$ & $-0.046$ \\
\bottomrule
\end{tabular}
\caption{Spearman $\rho$ between message length and correctness (left) and
failure rate (right), GPT-5.4 Mini Builder. \emph{Overall}: correlation across
all moves (the pooled row uses the per-model normalized length). Message
length is min--max normalized within each model, so 0 is that model's
shortest message and 1 its longest. Q1--Q4 split this normalized range into
four equal-width buckets: Q1 = 0--0.25 (shortest), Q2 = 0.25--0.5,
Q3 = 0.5--0.75, and Q4 = 0.75--1 (longest). Each Q column compares moves in
that bucket against all other moves; positive $\rho$ means the bucket has
higher values of the outcome than the rest. For correctness, positive $\rho$
means higher correctness scores; for failure rate, positive $\rho$ means more move failures.
$^{*}p<.05$, $^{**}p<.01$, $^{***}p<.001$.}
\label{tab:spearman_quantity_messagelength_gpt_5_4_mini}
\end{table*}

\subsubsection{Raw Utterance Lengths}

Table~\ref{tab:director_builder_stats} shows raw Director message length range (before normalization) by Director/Builder combination.




\begin{table}[ht]
\centering
\small
\resizebox{\linewidth}{!}{
\begin{tabular}{l ccc ccc}
\toprule
 & \multicolumn{3}{c}{\textbf{GPT-4.1-mini}} & \multicolumn{3}{c}{\textbf{GPT-5.4-mini}} \\
\cmidrule(lr){2-4} \cmidrule(lr){5-7}
\textbf{Director Model} & \textbf{Min} & \textbf{Max} & $n$ & \textbf{Min} & \textbf{Max} & $n$ \\
\midrule
GPT-5.4-Mini      & 13  & 41  & 196 & 15 & 44  & 224 \\
\midrule
Qwen 3.5 9B       & 65  & 601 & 238 & 8  & 564 & 223 \\
Qwen 3.5 9B SFT   & 27  & 385 & 306 & 26 & 411 & 326 \\
Qwen 3.5 9B DPO   & 26  & 390 & 295 & 26 & 433 & 291 \\
Qwen 3.5 9B IPO   & 152 & 371 & 311 & 24 & 371 & 294 \\
\midrule
Qwen 7B           & 12  & 87  & 354 & 11 & 41  & 352 \\
Qwen 7B SFT       & 22  & 63  & 336 & 26 & 79  & 320 \\
Qwen 7B DPO       & 16  & 50  & 361 & 15 & 50  & 353 \\
Qwen 7B IPO       & 21  & 112 & 359 & 22 & 54  & 361 \\
\midrule
Llama 8B          & 10  & 370 & 325 & -- & --  & --  \\
Llama 8B SFT      & 16  & 370 & 333 & -- & --  & --  \\
Llama 8B DPO      & 19  & 370 & 354 & -- & --  & --  \\
Llama 8B IPO      & 20  & 69  & 278 & -- & --  & --  \\
\bottomrule
\end{tabular}
}
\caption{Minimum, maximum, and total sample counts ($n$) for Director models evaluated across GPT-4.1-Mini and GPT-5.4-Mini Builder pairings. Not every Director model co-occured with both Builders.}
\label{tab:director_builder_stats}
\end{table}

Looking at the unnormalized ranges explains the failure spike seen in the \textit{middle} of Qwen3.5-9B in Fig.~\ref{fig:violations}, row 2, rather than strictly at the high extreme: Qwen3.5-9B base is notably more verbose at its extreme, meaning those spikes in the middle of the normalized range correspond to utterances about as long as the most verbose utterances of other models.

\subsection{``Color Burden'' as an Alternative Proxy for Quantity Adherence}
\label{app:colorburden}

\begin{figure}[h!]
    \centering
    \includegraphics[width=0.48\textwidth]{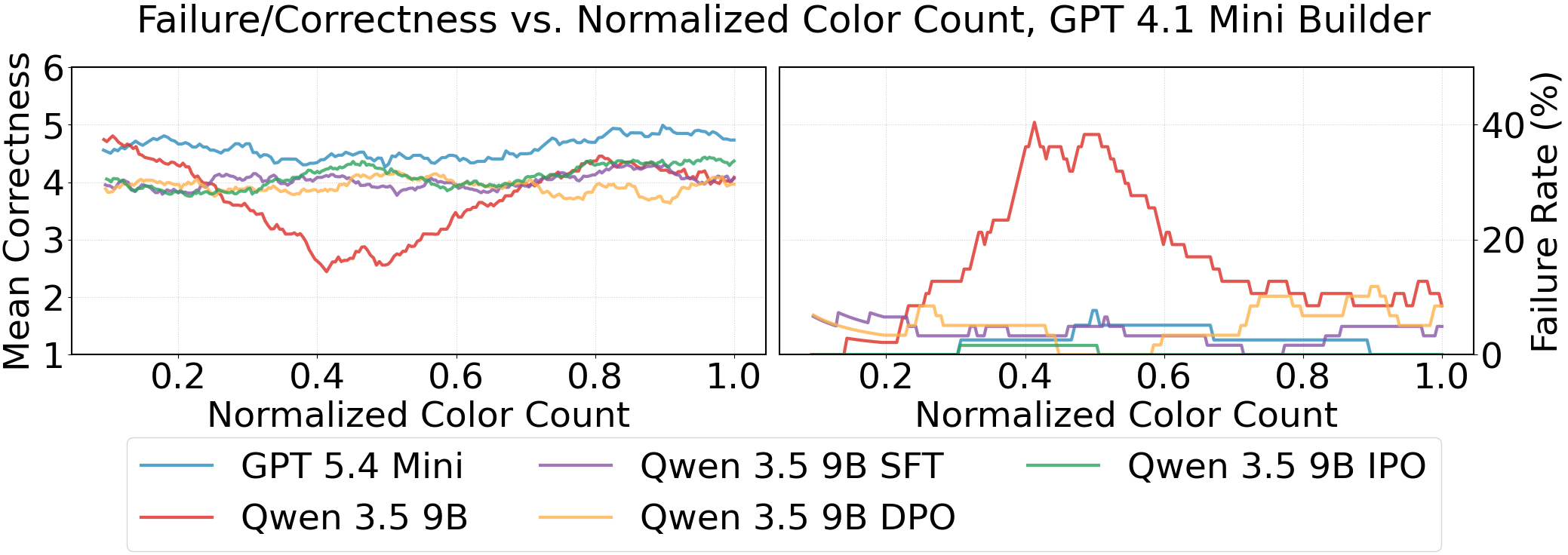}
    \includegraphics[width=0.48\textwidth]{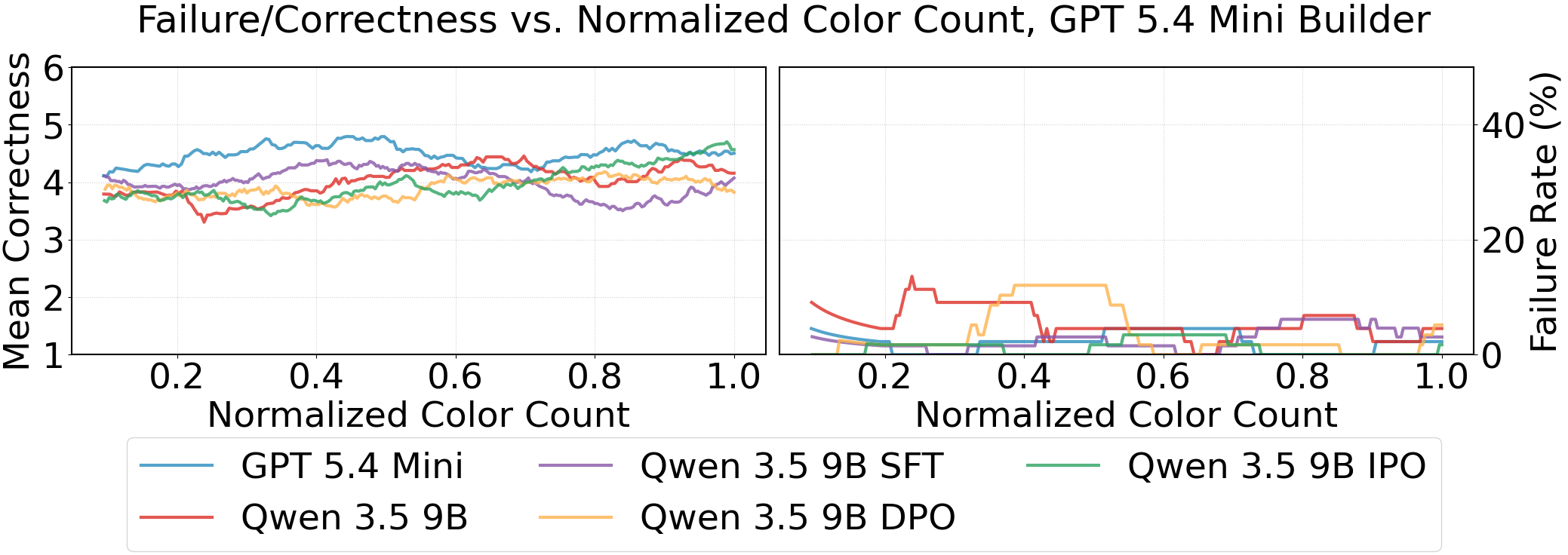}
    \caption{Move failure rate/mean correctness vs. normalized utterance-wise color token count by Director model.}\label{fig:violations-quantity-color-burden}
\end{figure}

The main paper presented results that assessed utterance length as a measure of Quantity adherence, and showed that greater Builder inferential capacity was associated with better recoverability of the underlying instruction meaning as instructions grew longer, and that aligned Directors produced utterances that resulted in lower move failure rates as utterance length grew. Table~\ref{fig:violations-quantity-color-burden} presents an alternative proxy for Quantity adherence---the ``color burden'' of an utterance, or normalized count of block color terms in an utterance, under the assumption that mentioning more blocks by color risks introducing more noise and padding into the utterance since the instruction is supposed to trigger only one block move, thus verging into overinformativity.  Table~\ref{fig:violations-quantity-color-burden} shows that the result holds over this proxy measurement as well: the GPT-5.4-mini Builder stably extracts utterance meaning across all models even as utterances grow longer. In the case of the weaker GPT-4.1-mini Builder, Director model alignment helps produce utterances with reduced failure rates, when compared to the base Qwen3.5-9B model.

\subsection{Cross-Builder Relation and Manner Comparison}
\label{app:buildercomp}

Tables~\ref{tab:maximpGpt41}--\ref{tab:mannerBuilderCom} show the GPT-4.1-mini vs. GPT-5.4-mini performance differences under different relation and manner flouting conditions. In particular we see that GPT-5.4-mini is affected by relation flouting in {\tt REMOVE} instructions to a statistically significant extent according to a paired $t$-test, in a way that GPT-4.1-mini is not. This is because GPT-4.1-mini struggles with {\tt REMOVE} instructions even when they are not relation flouts. While GPT-5.4-mini is stronger in this regard, under the non-flouting condition, it still displays pronounced sensitivity to relation flouting, showing that even models stronger in ``pragmatic listening'' struggle with the introduction of irrelevant information.

\begin{table*}[h!]
\vspace*{-2mm}
  \centering
  \resizebox{.98\linewidth}{!}{

\begin{tabular}{lcccccccccccc}
\toprule
 & \multicolumn{3}{c}{\textbf{Relation: remove}} & \multicolumn{3}{c}{\textbf{Relation: place}} & \multicolumn{3}{c}{\textbf{Relation: all actions}} & \multicolumn{3}{c}{\textbf{Manner: remove}} \\
\cmidrule(lr){2-4} \cmidrule(lr){5-7} \cmidrule(lr){8-10} \cmidrule(lr){11-13}
\textbf{Director} & Not Flouted & Flouted & $p$ & Not Flouted & Flouted & $p$ & Not Flouted & Flouted & $p$ & Not Flouted & Flouted & $p$ \\
\midrule
GPT 5.4 Mini & $54.3$ & $\mathbf{65.5}$ & 0.299 & $\mathbf{94.6}$ & $92.6$ & 0.653 & $\mathbf{80.9}$ & $78.6$ & 0.714 & $100.0$ & $100.0$ & 1.000 \\
Qwen 3.5 9b & $\mathbf{68.4}$ & $38.5$ & 0.149 & $\mathbf{79.5}$ & $54.5$ & 0.121 & $\mathbf{77.2}$ & $45.8$ & \underline{0.005} & $100.0$ & $100.0$ & 1.000 \\
Qwen 3.5 9b SFT & $\mathbf{58.1}$ & $22.7$ & \underline{0.009} & $\mathbf{96.0}$ & $60.0$ & \underline{0.037} & $\mathbf{78.5}$ & $29.6$ & \underline{$<0.001$} & $100.0$ & $100.0$ & 1.000 \\
Qwen 3.5 9b DPO & $\mathbf{50.0}$ & $46.4$ & 0.821 & $\mathbf{95.5}$ & $80.0$ & 0.123 & $\mathbf{74.0}$ & $55.3$ & \underline{0.043} & $\mathbf{100.0}$ & $90.5$ & 0.101 \\
Qwen 3.5 9b IPO & $\mathbf{29.8}$ & $23.8$ & 0.533 & $\mathbf{96.1}$ & $83.3$ & 0.288 & $\mathbf{54.8}$ & $31.2$ & \underline{0.007} & $100.0$ & $100.0$ & 1.000 \\
\midrule
Mean & $52.1$ & $39.4$ &  & $92.3$ & $74.1$ &  & $73.1$ & $48.1$ &  & $100.0$ & $98.1$  \\
$p_{\mathrm{Paired}}$ across models (pairs) &  &  & 0.217 \scriptsize{(5)} &  &  & \underline{0.034} \scriptsize{(5)} &  &  & \underline{0.031} \scriptsize{(5)} &  &  & 0.374 \scriptsize{(5)} \\
\bottomrule
\end{tabular}
}
\vspace*{-2mm}
    \caption{GPT 4.1 Builder performance across relation and manner flouting conditions.}
  \label{tab:maximpGpt41}
\vspace*{-2mm}
\end{table*}

\begin{table*}[h!]
  \centering
  \resizebox{.98\linewidth}{!}{

\begin{tabular}{lcccccccccccc}
\toprule
 & \multicolumn{3}{c}{\textbf{Relation: remove}} & \multicolumn{3}{c}{\textbf{Relation: place}} & \multicolumn{3}{c}{\textbf{Relation: all actions}} & \multicolumn{3}{c}{\textbf{Manner: remove}} \\
\cmidrule(lr){2-4} \cmidrule(lr){5-7} \cmidrule(lr){8-10} \cmidrule(lr){11-13}
\textbf{Director} & Not Flouted & Flouted & $p$ & Not Flouted & Flouted & $p$ & Not Flouted & Flouted & $p$ & Not Flouted & Flouted & $p$ \\
\midrule
GPT 5.4 Mini & $\mathbf{64.4}$ & $52.6$ & 0.238 & $\mathbf{92.6}$ & $65.5$ & \underline{$<0.001$} & $\mathbf{84.7}$ & $58.2$ & \underline{$<0.001$} & $100.0$ & $100.0$ & 1.000 \\
Qwen 3.5 9b & $\mathbf{91.7}$ & $50.0$ & 0.136 & $\mathbf{76.3}$ & $70.6$ & 0.752 & $\mathbf{78.9}$ & $66.7$ & 0.258 & $100.0$ & $100.0$ & 1.000 \\
Qwen 3.5 9b SFT & $\mathbf{65.4}$ & $60.0$ & 0.765 & $\mathbf{69.6}$ & $52.6$ & 0.264 & $\mathbf{68.3}$ & $56.4$ & 0.227 & $97.1$ & $\mathbf{100.0}$ & 1.000 \\
Qwen 3.5 9b DPO & $\mathbf{85.7}$ & $58.3$ & 0.096 & $\mathbf{94.4}$ & $75.0$ & \underline{0.035} & $\mathbf{91.5}$ & $67.9$ & \underline{0.003} & $\mathbf{96.2}$ & $94.7$ & 1.000 \\
Qwen 3.5 9b IPO & $\mathbf{76.3}$ & $50.0$ & 0.094 & $\mathbf{94.6}$ & $91.7$ & 0.633 & $\mathbf{87.2}$ & $76.3$ & 0.187 & $100.0$ & $100.0$ & 1.000 \\
\midrule
Mean & $76.7$ & $54.2$ &  & $85.5$ & $71.1$ &  & $82.1$ & $65.1$ &  & $98.7$ & $98.9$ &  \\
$p_{\mathrm{Paired}}$ across models (pairs) &  &  & \underline{0.024} \scriptsize{(5)} &  &  & \underline{0.032} \scriptsize{(5)} &  &  & \underline{0.007} \scriptsize{(5)} &  &  & 0.701 \scriptsize{(5)} \\
\bottomrule
\end{tabular}
}
\vspace*{-2mm}
    \caption{GPT 5.4 Builder performance across relation and manner flouting conditions.}
  \label{tab:maximpGpt54}
\vspace*{-2mm}
\end{table*}

\begin{table*}[h!]
  \centering
  \begin{tabular}{lcccc}
\toprule
 & \multicolumn{2}{c}{\textbf{Builder}} & \multicolumn{2}{c}{\textbf{GPT-5.4 $-$ GPT-4.1}} \\
\cmidrule(lr){2-3} \cmidrule(lr){4-5}
\textbf{Manner condition} & GPT-4.1 & GPT-5.4 & $\Delta$ & $p_{\mathrm{Welch}}$ \\
\midrule
Not flouted & $\mathbf{4.67_{\pm 0.12}}$ \scriptsize{(109)} & $4.45_{\pm 0.12}$ \scriptsize{(114)} & $-0.22$ & 0.183 \\
Not flouted (accepted dup.) & $\mathbf{4.63_{\pm 0.17}}$ \scriptsize{(57)} & $4.41_{\pm 0.15}$ \scriptsize{(69)} & $-0.22$ & 0.318 \\
Flouted & $4.47_{\pm 0.15}$ \scriptsize{(87)} & $\mathbf{4.50_{\pm 0.12}}$ \scriptsize{(110)} & $+0.03$ & 0.872 \\
\bottomrule
\end{tabular}
\vspace*{-2mm}
    \caption{Builder comparison across manner flouting conditions.}
  \label{tab:mannerBuilderCom}
\vspace*{-2mm}
\end{table*}

\subsection{Director Model Alignment}

\begin{figure}[h!]
    \centering
    \includegraphics[width=\linewidth]{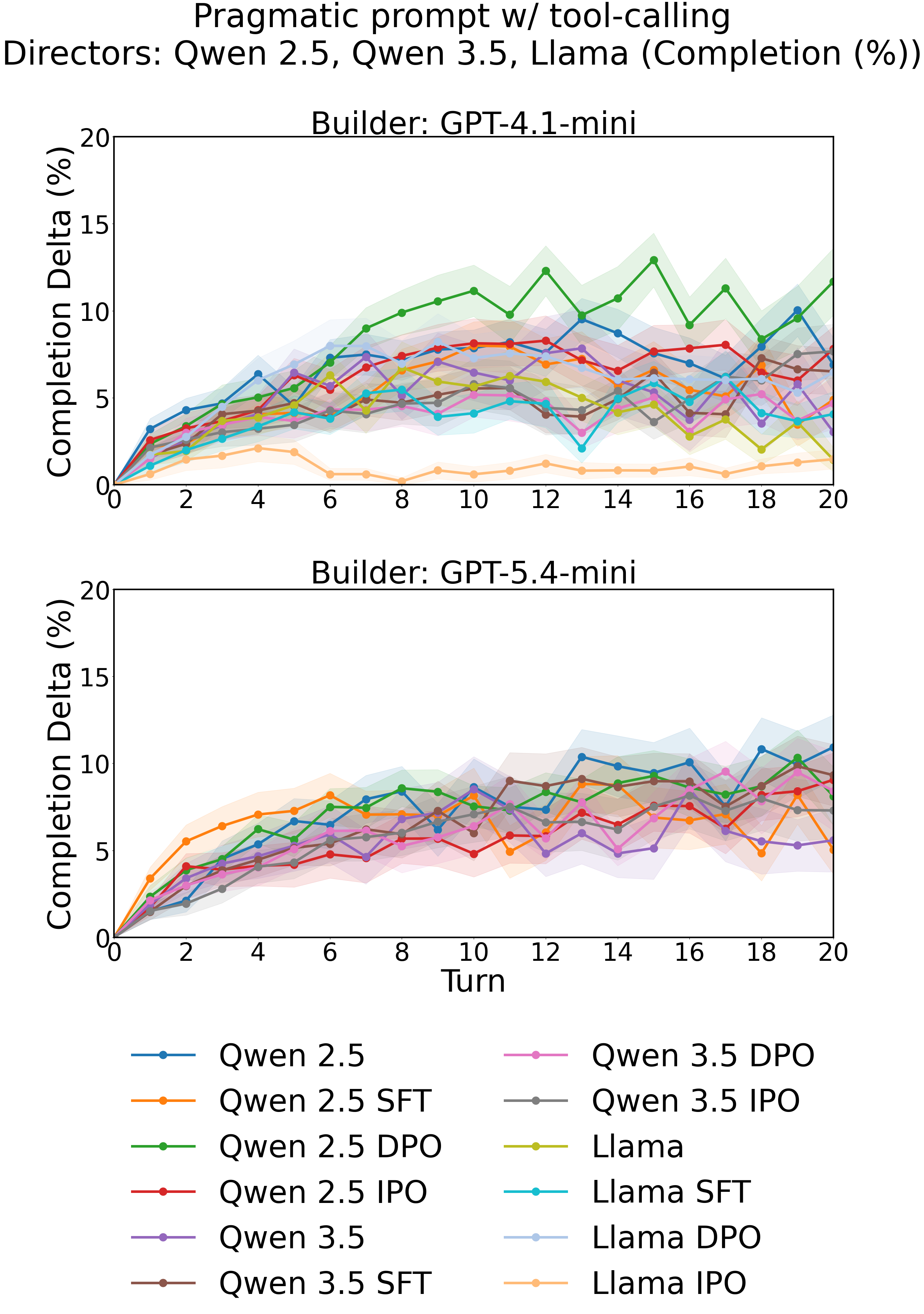}
\vspace*{-2mm}
    \caption{Effects of Director post-training with pragmatic prompting and move exploration, vs. GPT-4.1-mini and GPT-5.4-mini Builders.}
    \label{fig:posttraining}
\vspace*{-2mm}
\end{figure}

Fig.~\ref{fig:posttraining} shows post-trained open models (Qwen2.5-7B, Qwen3.5-9B, and Llama 3-8B) as Directors in games starting from an empty board, by Builder prompt condition. Against a GPT-5.4-mini Builder, overall alignment effects are muted, usually within the standard error of the mean. This is consistent with the finding that baseline GPT-5.4-mini already behaves more like a pragmatic Builder with greater inferential capacity than a literal Builder---it is able to extract actionable meaning where it is present but obscured by Director utterance style. Effects are more pronounced against the GPT-4.1-mini Builder, which is in part expected. The preference dataset $\mathcal{D}$ was itself generated with GPT-4.1-mini; aligning against this data makes the Director model generate utterances that are more like those of GPT-4.1-mini.

Supervised fine-tuning of the Director model often slightly hurts performance relative to the base model---``mean-seeking'' \citep{chan2022greedification} SFT may produce ``safe'' actionable-sounding utterances decoupled from actual task state. However, DPO alignment brings it back up when the Builder is explicitly Literal or, in particular, Pragmatic. This shows the convergence between more optimal speakers and pragmatic listeners predicted by RSA, alignment toward optimal utterances may provide robustness against listener suboptimalities, with Qwen as a stronger Director overall than Llama. Llama with IPO in particular seems to repeatedly place and remove the same few blocks, creating a correction spiral that inhibits progress. IPO's stronger KL-regularization may keep it closer to the underperforming SFT reference model, creating a feedback loop. The relatively small preference dataset $\mathcal{D}$ may also be a factor here. Llama's overall underperformance relative to Qwen, in which differences often collapsed to within statistical noise, motivated the use of Qwen3.5-9B in the main experiments. Due to this, Llama 3-8B was also not evaluated against a GPT-5.4-mini Builder.

\subsection{Qwen2.5-7B and Llama 3-8B Directors}

\begin{figure*}[h!]
\vspace*{-2mm}
    \centering
     \begin{subfigure}[b]{\textwidth}
        \centering
        \includegraphics[width=0.96\textwidth]{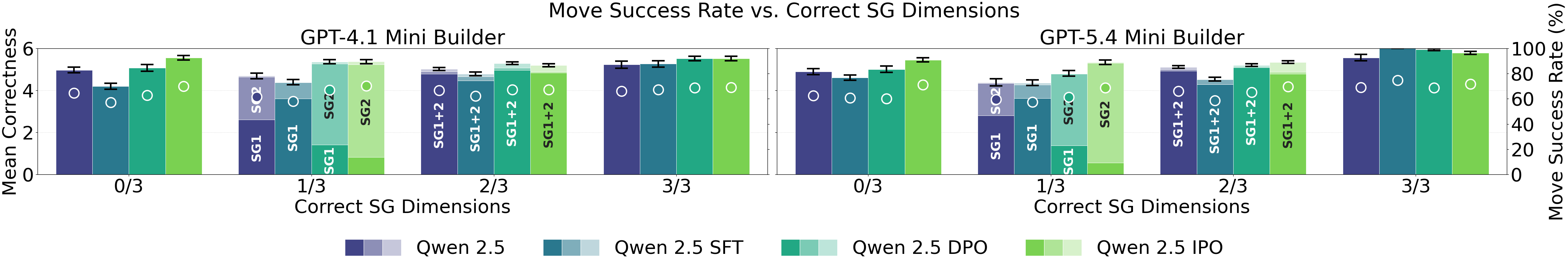}
        \label{fig:violations-quality25}
    \end{subfigure}
    \begin{subfigure}[b]{\textwidth}
        \centering
        \includegraphics[width=0.96\textwidth]{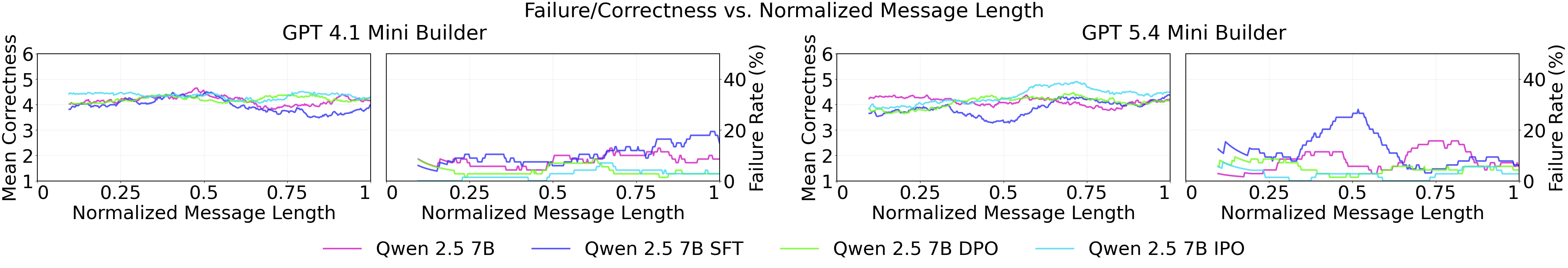}
        \label{fig:violations-quantity25}
    \end{subfigure}
    \begin{subfigure}[b]{\textwidth}
        \centering
        \includegraphics[width=0.96\textwidth]{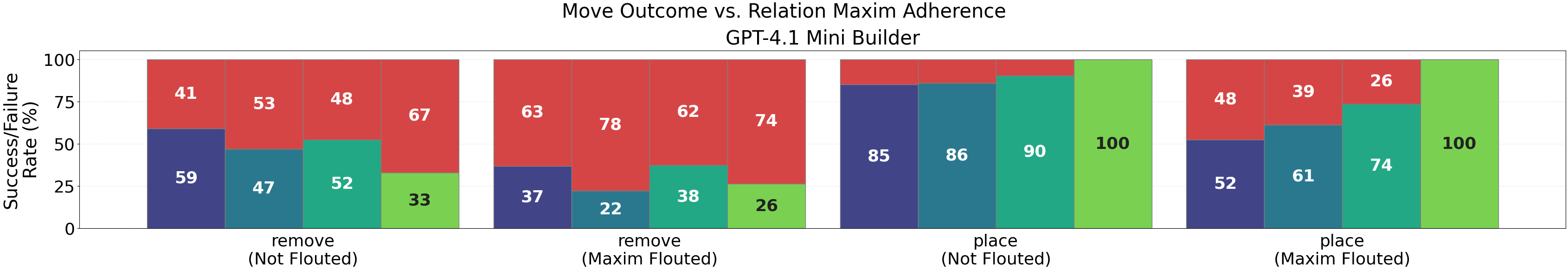}
        \label{fig:violations-relation25}
    \end{subfigure}
    \begin{subfigure}[b]{\textwidth}
        \centering
        \includegraphics[width=0.96\textwidth]{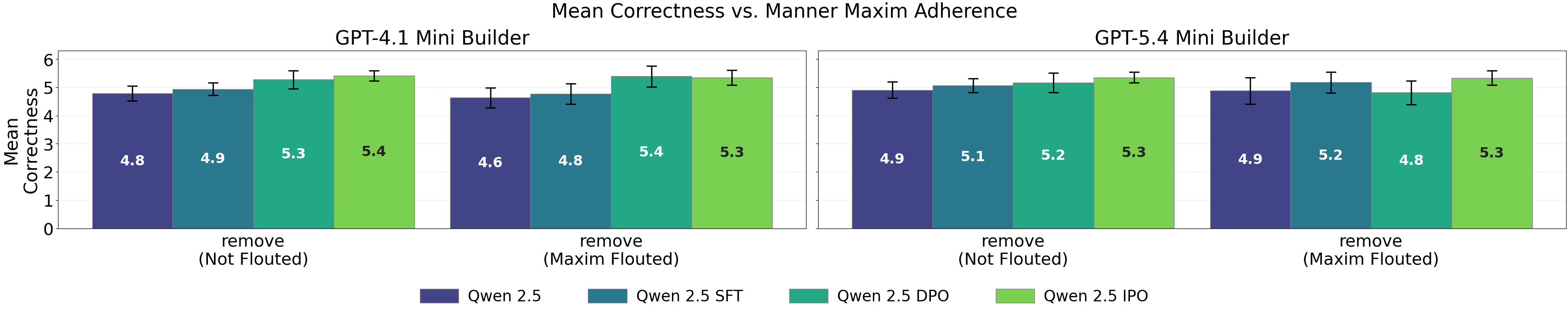}
        \label{fig:violations-manner25}
    \end{subfigure}
\vspace*{-6mm}
    \caption{\label{fig:violations-qwen25}Top to bottom: 1) Move success rate vs. correct grounding dimensions in Director utterances. 
    \textbf{Shades} in columns 1/3 and 2/3 indicate the proportion of total correct SG elements belonging to each element or subset. \textbf{Circles} indicate mean move correctness.
    2) Move failure rate/mean correctness vs. normalized message length by Director model. 3) {\tt PLACE} and {\tt REMOVE} outcomes vs. maxim of relation adherence. 4) {\tt REMOVE} outcomes vs. maxim of manner adherence. Builder uses Pragmatic prompting and move exploration.}
\vspace*{-4mm}
\end{figure*}

\begin{figure*}[h!]
\vspace*{-2mm}
    \centering
     \begin{subfigure}[b]{.48\textwidth}
        \centering
        \includegraphics[width=\textwidth]{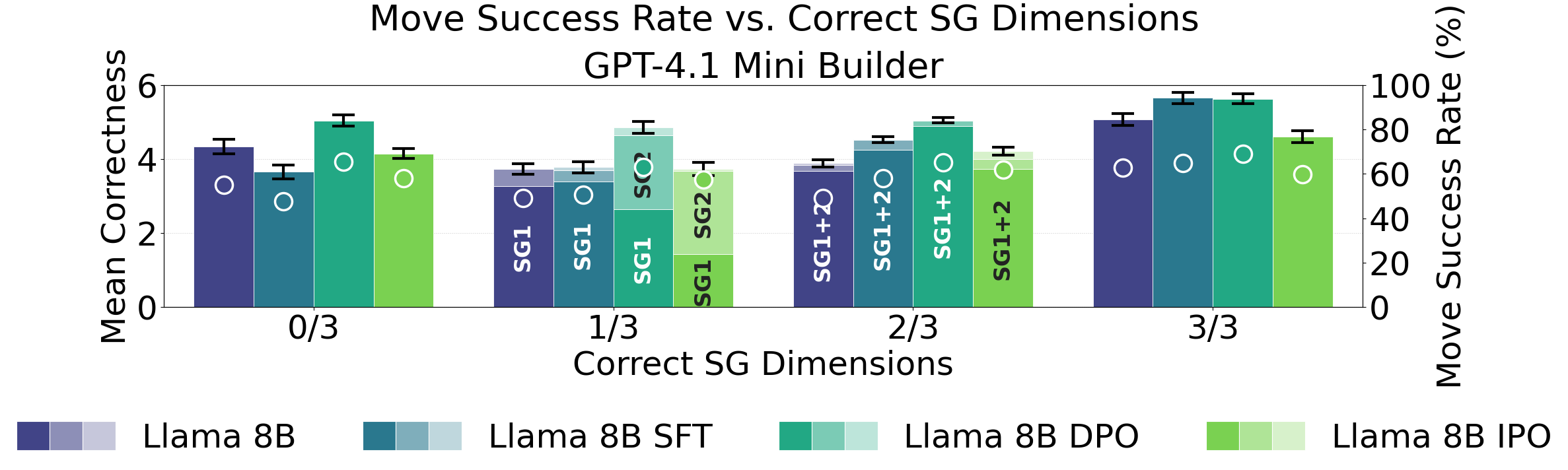}
        \label{fig:violations-qualityLlama}
    \end{subfigure}
    \begin{subfigure}[b]{.48\textwidth}
        \centering
        \includegraphics[width=\textwidth]{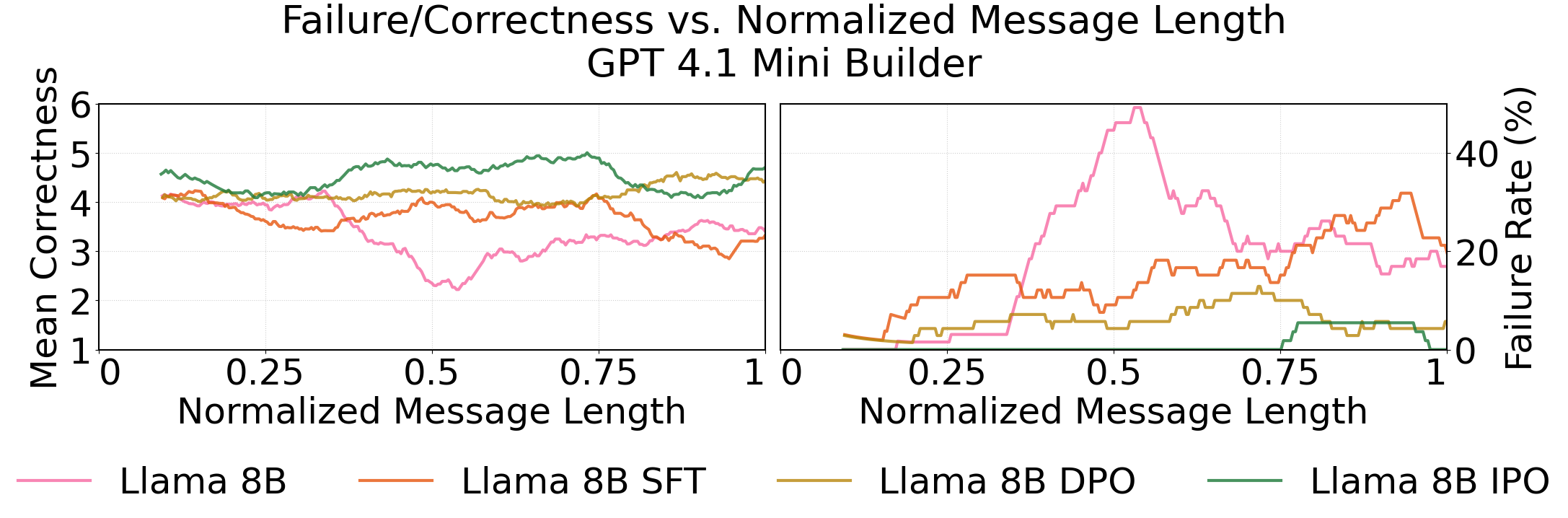}
        \label{fig:violations-quantityLlama}
    \end{subfigure}
    \begin{subfigure}[b]{\textwidth}
        \centering
        \includegraphics[width=0.96\textwidth]{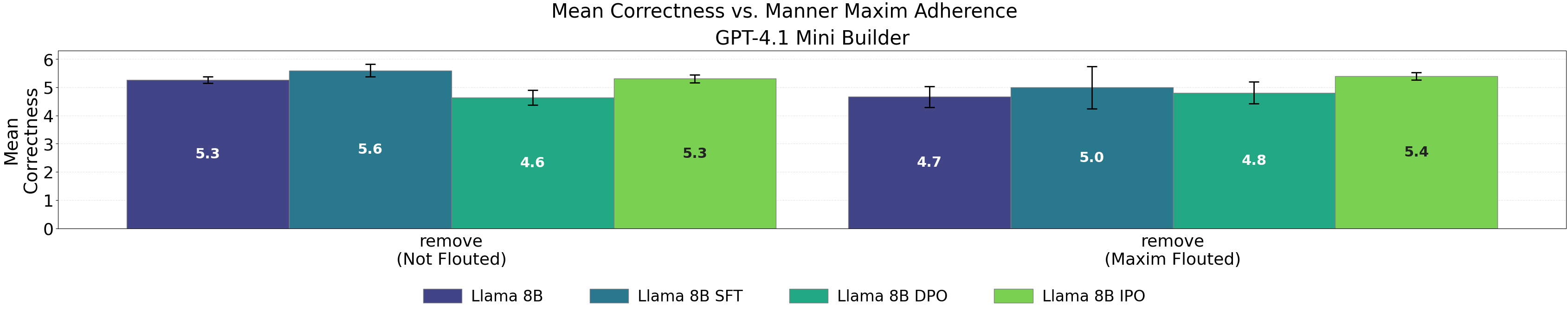}
        \label{fig:violations-mannerLlama}
    \end{subfigure}
\vspace*{-6mm}
    \caption{\label{fig:violations-llama}Top to bottom: 1) Move success rate vs. correct grounding dimensions in Director utterances. 
    \textbf{Shades} in columns 1/3 and 2/3 indicate the proportion of total correct SG elements belonging to each element or subset. \textbf{Circles} indicate mean move correctness.
    2) Move failure rate/mean correctness vs. normalized message length by Director model. 3) {\tt REMOVE} outcomes vs. maxim of manner adherence. Builder uses Pragmatic prompting and move exploration.}
\vspace*{-4mm}
\end{figure*}

Figs.~\ref{fig:violations-qwen25} and~\ref{fig:violations-llama} show additional analyses of failure modes by Gricean maxim for Qwen2.5-7B-Instruct and Llama 3.1-8B-Instruct, respectively. These figures complement Fig.~\ref{fig:violations} in the main paper. 

\section{Sensitivity to Prompt Phrasing}
\label{app:prompt-sensitivity}

Tables~\ref{tab:promptsens1} and \ref{tab:promptsens2}, and Figs.~\ref{fig:combinedPrompt1}--\ref{fig:combinedPrompt3} show results from tests of a subsample of structures (1 per start condition) using the Literal and Pragmatic Builder prompt variants (Appendix~\ref{app:prompts}).

\begin{table}[h!]
  \centering
  \resizebox{\linewidth}{!}{
  \begin{tabular}{lcccc}
    \toprule
   \textbf{Prompt Variant} & \textbf{Mean} & \textbf{SEM} & \textbf{Min} & \textbf{Max} \\
    \midrule
    Literal 2 w/ tools  & 0.017 & 0.019 & -0.105 & 0.088 \\
    Literal 2 w/o tools & -0.019 & 0.014 & -0.128 & 0.056 \\
    Pragmatic 2 w/ tools & 0.019 & 0.023 & -0.072 & 0.186 \\
    Pragmatic 2 w/o tools & -0.037 & 0.037 & -0.418 & 0.081 \\
    Literal 3 w/ tools & 0.056 & 0.016 & 0.003 & 0.198 \\
    Literal 3 w/o tools & -0.011 & 0.012 & -0.085 & 0.044 \\
    Pragmatic 3 w/ tools & 0.014 & 0.014 & -0.035 & 0.101 \\
    Pragmatic 3 w/o tools & -0.039 & 0.036 & -0.418 & 0.048 \\
    \bottomrule
  \end{tabular}}
  \caption{Overall Progress $\Delta$ from Start (Builder: GPT-4.1-mini, Director: GPT-4.1-mini, 1 structure per start condition)}
  \label{tab:promptsens1}
\end{table}

\begin{table}[h!]
  \centering
  \resizebox{\linewidth}{!}{
  \begin{tabular}{lcccc}
    \toprule
   \textbf{Prompt Variant} & \textbf{Mean} & \textbf{SEM} & \textbf{Min} & \textbf{Max} \\
    \midrule
     Literal 2 w/ tools  & 0.045 & 0.016 & -0.043 & 0.136 \\
     Literal 2 w/o tools & -0.008 & 0.017 & -0.130 & 0.091 \\
     Pragmatic 2 w/ tools & 0.068 & 0.023 & -0.043 & 0.227 \\
     Pragmatic 2 w/o tools & 0.006 & 0.015 & -0.095 & 0.080 \\
     Literal 3 w/ tools & 0.093 & 0.021 &  0.000 & 0.273 \\
     Literal 3 w/o tools & 0.007 & 0.013 & -0.048 & 0.083 \\
     Pragmatic 3 w/ tools & 0.048 & 0.015 &  0.000 & 0.182 \\
     Pragmatic 3 w/o tools & 0.006 & 0.015 & -0.095 & 0.087 \\
    \bottomrule
  \end{tabular}}
  \caption{Completion \% $\Delta$ from Start (Builder: GPT-4.1-mini, Director: GPT-4.1-mini, 1 structure per start condition)}
  \label{tab:promptsens2}
\end{table}

\begin{figure*}[h!]
    \centering
    \includegraphics[width=\linewidth]{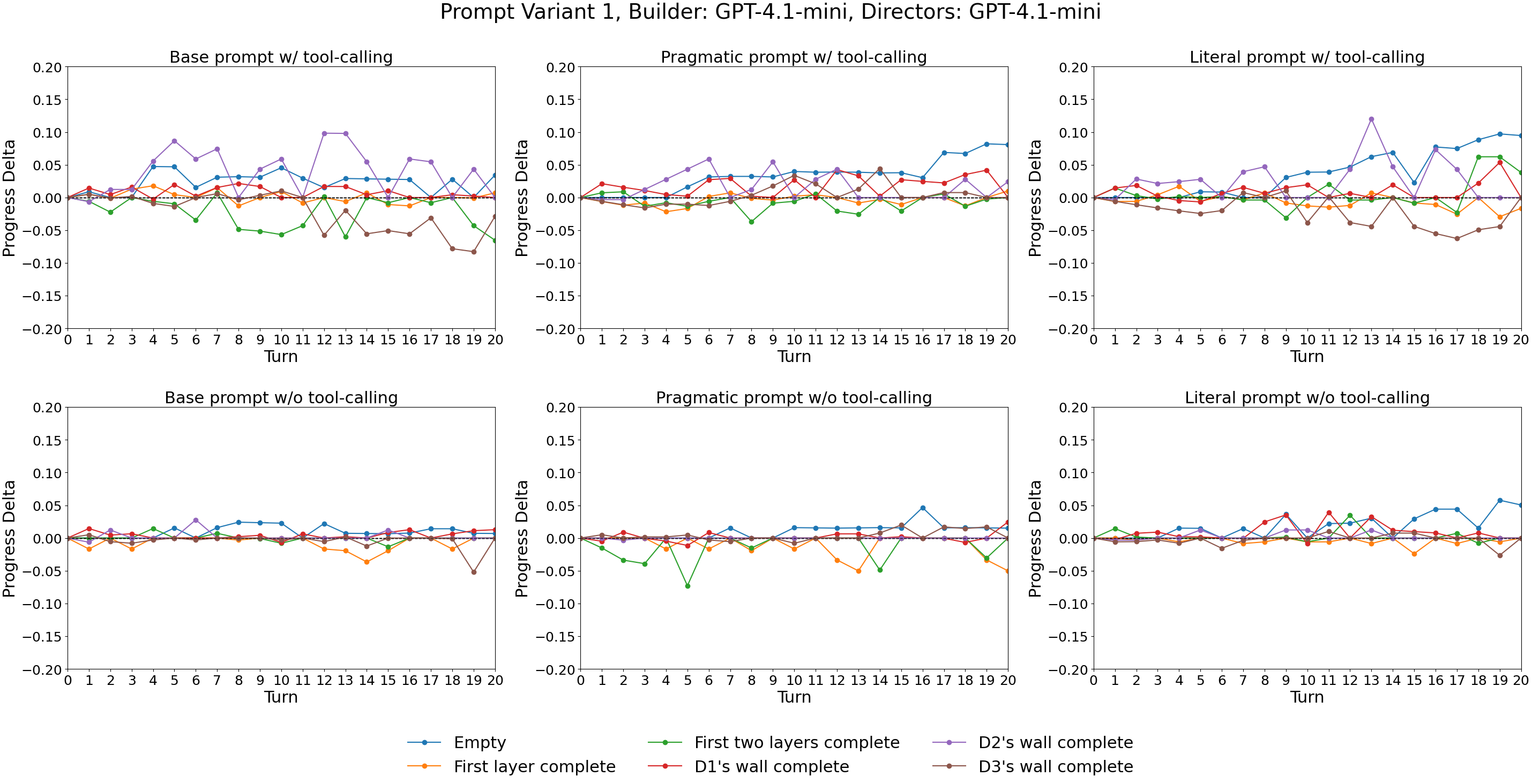}
    \caption{Prompting and tool-calling outcomes results with GPT-4.1-mini as Builder and GPT-4.1-mini as Directors, using Literal and Pragmatic Builder prompt variant 1 (same as versions used in main experiments).}
    \label{fig:combinedPrompt1}
\end{figure*}

\begin{figure*}[h!]
    \centering
    \includegraphics[width=\linewidth]{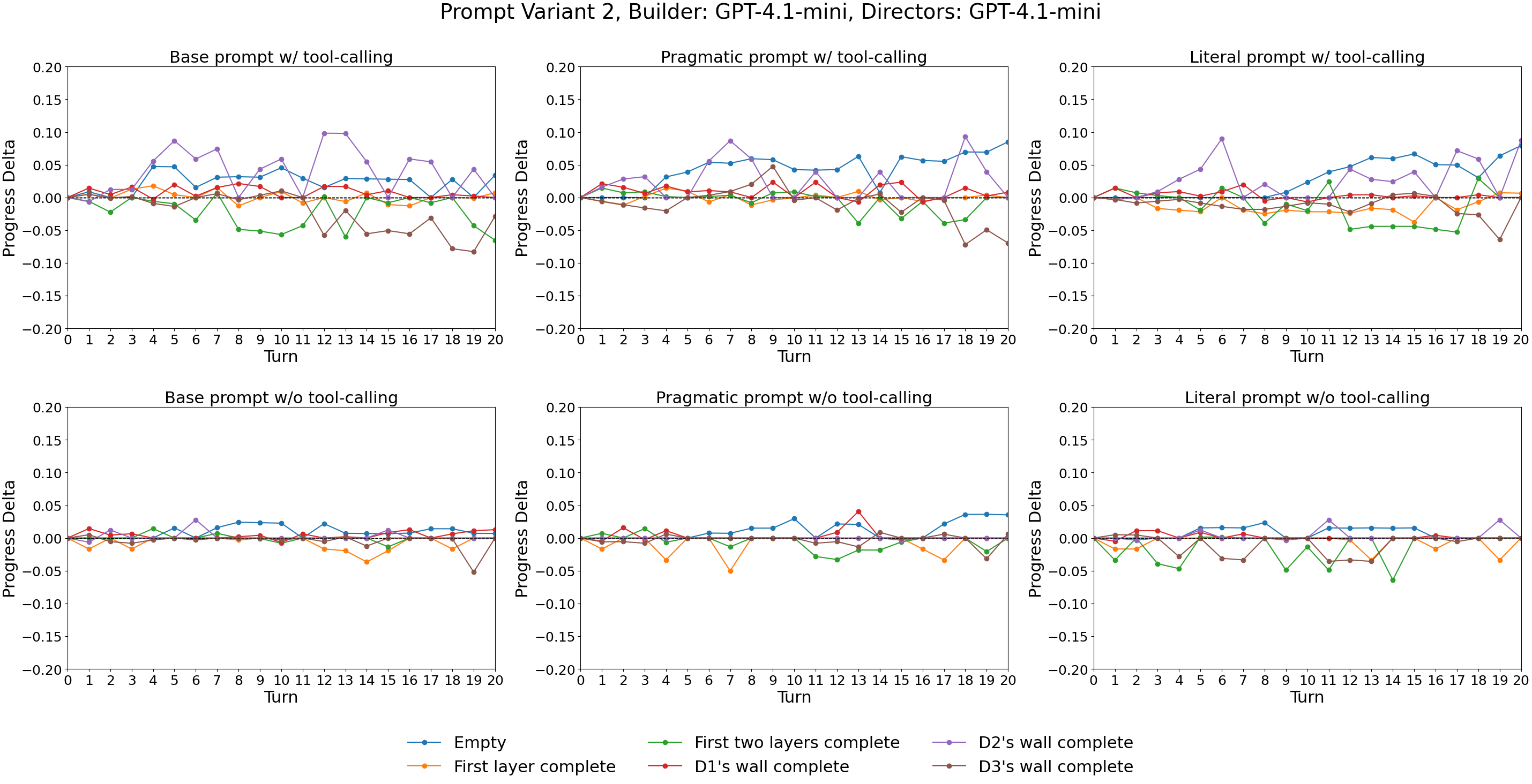}
    \caption{Prompting and tool-calling outcomes results with GPT-4.1-mini as Builder and GPT-4.1-mini as Directors, using Literal and Pragmatic Builder prompt variant 2 (Appendix~\ref{app:prompts}).}
    \label{fig:combinedPrompt2}
\end{figure*}

\begin{figure*}[h!]
    \centering
    \includegraphics[width=\linewidth]{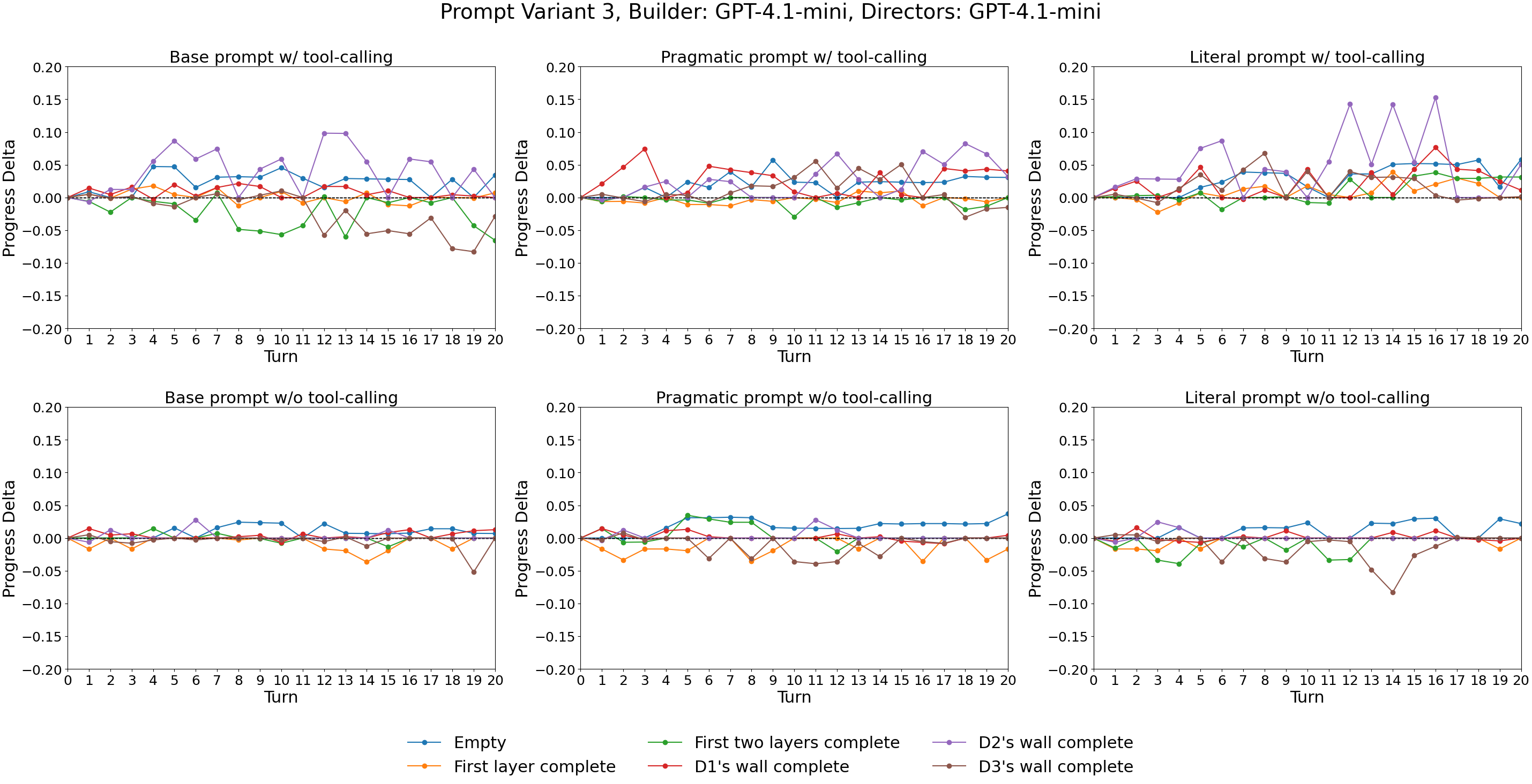}
    \caption{Prompting and tool-calling outcomes results with GPT-4.1-mini as Builder and GPT-4.1-mini as Directors, using Literal and Pragmatic Builder prompt variant 3 (Appendix~\ref{app:prompts}).}
    \label{fig:combinedPrompt3}
\end{figure*}

These results show that trends are preserved with prompt rephrasing and that overall variance across rephrasings remains low.

\section{AI Use Declaration}
\label{app:ai-use}

Portions of the code used for statistical analysis, figure generation, and LaTeX table production were developed with the assistance of Claude Opus 5.5 (Anthropic). All code and outputs were reviewed and verified by the authors.

\end{document}